\documentclass[10pt,twocolumn,letterpaper]{article}

\usepackage{cvpr}              
\definecolor{cvprblue}{rgb}{0.21,0.49,0.74}
\usepackage[pagebackref,breaklinks,colorlinks,allcolors=cvprblue]{hyperref}
\usepackage{amsthm}
\usepackage{amsmath}
\usepackage{amsfonts}
\usepackage{colortbl}
\usepackage{multirow}
\newtheorem{theorem}{Theorem}
\theoremstyle{definition}
\newtheorem{definition}{Definition}
\usepackage[table,xcdraw]{xcolor}
\usepackage{arydshln}
\usepackage{tikz}
\usepackage[normalem]{ulem}
\useunder{\uline}{\ul}{}
\usepackage{booktabs}
\usepackage{makecell}
\usepackage[table]{xcolor}   
\usepackage[accsupp]{axessibility} 
\usepackage{hyperref}

\def\paperID{36132} 
\def\confName{CVPR}
\def\confYear{2026}

\title{\textsc{HyCal}: A Training-Free Prototype Calibration Method for\\Cross-Discipline Few-Shot Class-Incremental Learning}

\author{
Eunju Lee$^{*}$\\
Chung-Ang University\\
{\tt\small dmswn5829@cau.ac.kr}
\and
MiHyeon Kim$^{*}$\\
KT Corporation\\
{\tt\small mihyeon.kim@kt.com}
\and
JuneHyoung Kwon, Yoonji Lee\\
Chung-Ang University\\
{\tt\small \{dirchdmltnv,pioneer0305\}@cau.ac.kr}
\and
JiHyun Kim\\
Dentium\\
{\tt\small jhkim20@dentium.com}
\and
Soojin Jang\\
ETRI\\
{\tt\small soojin@etri.re.kr}
\and
YoungBin Kim$^{\dagger}$\\
Chung-Ang University\\
{\tt\small ybkim85@cau.ac.kr}
}

\begin{document}
\maketitle
\begingroup
\renewcommand\thefootnote{}
\footnotetext{$^{*}$ Equal contribution. $^{\dagger}$ Corresponding author.}
\endgroup

\begin{abstract}
Pretrained Vision–Language Models (VLMs) like CLIP show promise in continual learning, but existing Few-Shot Class-Incremental Learning (FSCIL) methods assume homogeneous domains and balanced data distributions, limiting real-world applicability where data arises from heterogeneous disciplines with imbalanced sample availability and varying visual complexity. We identify \textit{Domain Gravity}, a representational asymmetry where data imbalance across heterogeneous domains causes overrepresented or low-entropy domains to disproportionately influence the embedding space, leading to prototype drift and degraded performance on underrepresented or high-entropy domains. To address this, we introduce \textbf{Cross-Discipline Variable Few-Shot Class-Incremental Learning (XD-VSCIL)}, a benchmark capturing real-world heterogeneity and imbalance where Domain Gravity naturally intensifies. We propose \textbf{Hybrid Prototype Calibration (\textsc{HyCal})}, a training-free method combining cosine similarity and Mahalanobis distance to capture complementary geometric properties—directional alignment and covariance-aware magnitude—yielding stable prototypes under imbalanced heterogeneous conditions. Operating on frozen CLIP embeddings, \textsc{HyCal} achieves consistent retention–adaptation improvements while maintaining efficiency. Experiments show \textsc{HyCal} effectively mitigates Domain Gravity and outperforms existing methods in imbalanced cross-domain incremental learning. Code: \href{https://github.com/EJLEE5826/HyCal-CIL} {Link}
\end{abstract}
    
\section{Introduction}
\label{sec:intro}
\begin{figure*}[t]
\centering
\includegraphics[width=0.95\textwidth]{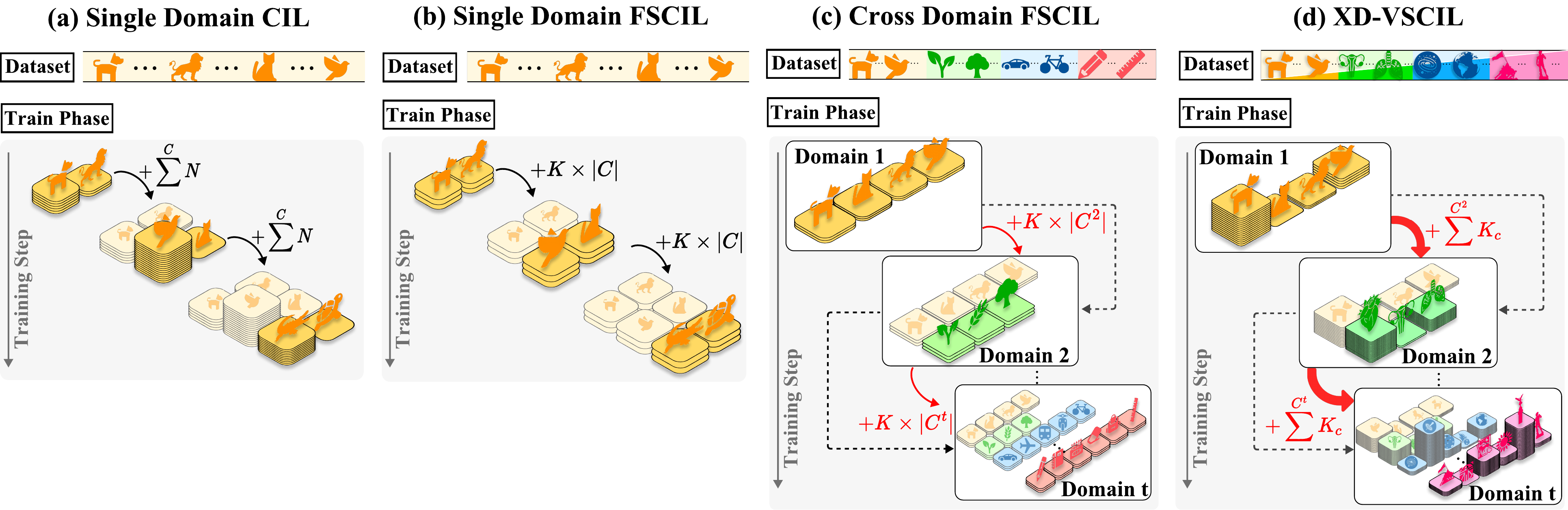}
\caption{Overview of continual learning paradigms.
(a) CIL: Class-incremental learning without domain shift.
(b) FSCIL: Fixed shots per class within a single domain.
(c) Cross-domain FSCIL: Multiple domains with fixed shots and uniform class/sample sizes.
(d) XD-VSCIL (ours): Realistic setting with sequential tasks from diverse domains, varying class counts, and data imbalance.}
\label{fig: teaser}
\end{figure*}

Pretrained Vision–Language Models (VLMs) such as CLIP demonstrate strong continual learning performance across diverse domains~\cite{LI2025103404,sun2025mos,guo2024information,ma2024pip}.
Standard Class-Incremental Learning (CIL), shown in~\cref{fig: teaser} (a), assumes abundant and balanced data without domain shifts. Few-Shot CIL (FSCIL), illustrated in~\cref{fig: teaser} (b), relaxes this assumption by limiting each class to a fixed number of samples~\cite{tao2020few,goswami2024calibrating,feng2025defying}, yet most FSCIL frameworks still presuppose homogeneous visual domains and balanced few-shot distributions~\cite{tao2020few,goswami2024calibrating}.
In real-world settings, data often arises from heterogeneous domains—such as medical imagery, textures, aerial views, or natural scenes—that differ markedly in visual statistics and representational complexity.
These practical considerations have motivated recent extensions of FSCIL toward cross-domain scenarios, as shown in~\cref{fig: teaser} (c), leveraging VLMs’ zero-shot generalization to retain prior knowledge across diverse domains~\cite{xu2024advancing,yu2024boosting,luo2025lada}.

However, most cross-domain FSCIL frameworks still assume fixed few-shot configurations and balanced data distributions~\cite{tao2020few,goswami2024calibrating}, limiting their relevance to real-world conditions. In practice, domain heterogeneity and data imbalance are pervasive: heterogeneous domains differ in visual entropy, feature geometry, and intra-domain consistency, while data-rich or visually coherent domains tend to dominate learning. This combination leads to overfitting, catastrophic forgetting, and ultimately intrinsic representational imbalance, where dominant domains exert disproportionately strong influence on the embedding space and open-world or high-entropy domains remain structurally underrepresented~\cite{10.5555/3524938.3525120,DBLP:journals/tmlr/HuangLYWH24,xu2024advancing,goswami2023fecam}. These issues are further intensified in highly diverse fields—such as healthcare or aerospace—where distribution shifts are large and data availability varies substantially across tasks~\cite{xu2024advancing,yu2024boosting}.

We identify a structural bias emerging when VLMs encounter heterogeneous domain distributions, termed \textit{Domain Gravity}—a representational asymmetry driven by domain-specific statistical properties that accumulate throughout learning~\cite{CLIP,gao2024clip}.

\begin{definition}(Domain Gravity).\\
\textit{Each domain induces a representational potential based on its visual coherence and data density.
Overrepresented or low-entropy domains exert disproportionately strong influence in the shared embedding space.}
\end{definition}

Domain Gravity manifests in two forms. First, pretrained VLMs such as CLIP inherit distributional biases from large-scale corpora, where visually dominant or frequent domains disproportionately shape the embedding geometry~\cite{CLIP}. This produces weaker and less stable representations for \textit{minority} or \textit{high-entropy} domains~\cite{wen2024makes}, making it harder to exploit VLMs’ rich prior knowledge and zero-shot capability as visual disciplines specialize. Second, this bias intensifies when incremental tasks arise from domains with varying entropy, visual structure, or data availability~\cite{xu2024advancing,yu2024boosting}. In cross-domain FSCIL, prototypes and embeddings drift toward visually coherent or data-rich domains, degrading representation quality for sparse or high-entropy domains. Although many methods address class imbalance or sample scarcity, they implicitly assume homogeneous feature distributions and stable covariance estimation~\cite{goswami2023fecam,goswami2024calibrating,tao2020few}. Under \textit{heterogeneous} domain priors, where domains differ in scale, variance, and orientation, metric calibration becomes biased, and covariance-based regularization becomes unreliable in few-shot regimes~\cite{talur2023few}.
Consequently, existing methods fail to counteract the asymmetric representational forces that define Domain Gravity.
To capture these real-world challenges, we adopt \textbf{Cross-Discipline Variable Few-Shot Class-Incremental Learning (XD-VSCIL)}, where each task originates from a distinct visual discipline with varying entropy and sample availability. Domain Gravity intensifies in this setting, making XD-VSCIL a practical benchmark for evaluating robustness under heterogeneous and imbalanced conditions.

To address these challenges, we propose \textbf{Hybrid Prototype Calibration (\textsc{HyCal})}, a training-free method operating on frozen pretrained embeddings. \textsc{HyCal} combines cosine similarity, capturing global directional structure, with Mahalanobis distance, encoding domain-specific covariance and entropy differences. This hybrid calibration yields stable prototypes under heterogeneous domain shifts without additional parameters or fine-tuning. \textsc{HyCal} mitigates the representational asymmetry induced by Domain Gravity, yielding consistent improvements in the retention–adaptation trade-off across heterogeneous disciplines. Its training-free design—relying solely on stored embeddings and distance computations—enables seamless integration of new visual disciplines without altering model parameters, maintaining both efficiency and interpretability.

\begin{figure*}[t]
    \centering

    \begin{minipage}[c]{0.30\linewidth}
        \centering
        \scriptsize
        \setlength{\tabcolsep}{1.0mm}
        \renewcommand{\arraystretch}{0.95}
        \captionof{table}{Performance comparison under various few-shot settings (Last Acc., \%).}
        \resizebox{\linewidth}{!}{%
        \begin{tabular}{lllll}
        \toprule
        \begin{tabular}[c]{@{}l@{}}\textbf{Domain}\\ (\#Class)\end{tabular} &
        \begin{tabular}[c]{@{}l@{}}\textbf{Med}\\ (11)\end{tabular} &
        \begin{tabular}[c]{@{}l@{}}\textbf{Tex}\\(47)\end{tabular} &
        \begin{tabular}[c]{@{}l@{}}\textbf{Last}\\ \textbf{Acc.}\end{tabular} &
        \begin{tabular}[c]{@{}l@{}}\textbf{Gap}\\ ($\Delta\downarrow$)\end{tabular} \\
        \midrule
        Zero-shot & 22.4 & 44.3 & -- & 21.9 \\
        \midrule 
        \multicolumn{5}{l}{\textbf{General Few-Shot (10-Shot)}} \\
        \textit{\# of Training Samples} & \textit{110} & \textit{470} & - & \textit{370} \\
        \quad {Primal-RAIL~\cite{xu2024advancing}}        & {\textcolor{gray}{54.65}} & {\textcolor{gray}{66.37}} & {\textcolor{gray}{60.51}} & {\textcolor{gray}{11.72}} \\ \hdashline
          \quad RanPAC~\cite{mcdonnell2023ranpac}                & {62.54} & {64.60} & {63.57} & 2.06 \\
          \quad KLDA~\cite{momeni2025continual}              & 3.46 &
                                 0.98 &
                                  1.23 &
                                 2.48 \\
          \quad \textsc{HyCal} (Ours)           & \textbf{65.04} & \textbf{65.48} & \textbf{65.26} & \textbf{0.45} \\
        \midrule
        \multicolumn{5}{l}{\textbf{Balanced (20-shot vs 5-shot)}} \\
        \textit{\# of Training Samples} & \textit{220} & \textit{235} & - & \textit{15} \\
        \quad {Primal-RAIL~\cite{xu2024advancing}}       & {\textcolor{gray}{67.69}} & {\textcolor{gray}{61.64}} & {\textcolor{gray}{64.67}} & {\textcolor{gray}{6.05}} \\  \hdashline
        \quad RanPAC~\cite{mcdonnell2023ranpac}                & {66.51} & {55.02} & {60.77} & 11.49 \\
        \quad KLDA~\cite{momeni2025continual}                   & 4.20 & 0.99 & 0.51& \textbf{3.21\textsuperscript{*}} \\
        \quad \textsc{HyCal} (Ours)           & \textbf{69.38} & \textbf{60.58} & \textbf{64.98} & 8.80 \\
        \midrule
        \multicolumn{5}{l}{\textbf{Imbalanced (5-shot vs 10-shot)}} \\
        \textit{\# of Training Samples} & \textit{55} & \textit{470} & - & \textit{415} \\
        \quad {Primal-RAIL~\cite{xu2024advancing}}            & {\textcolor{gray}{56.28}} & {\textcolor{gray}{65.31}} & {\textcolor{gray}{60.79}} & {\textcolor{gray}{9.03}} \\ \hdashline
        \quad RanPAC~\cite{mcdonnell2023ranpac}                & {5.82} & {62.65} & {59.23} & {6.83} \\
        \quad KLDA~\cite{momeni2025continual}                 & 4.73& 0.37& 1.83 & \textbf{4.36\textsuperscript{*}}\\
        \quad \textsc{HyCal} (Ours)           & \textbf{60.37} & \textbf{65.31} & \textbf{62.84} & 4.94 \\
        \bottomrule
        \end{tabular}
        
        }

        \label{tab:freeze_train_gap}
    \end{minipage}
    \hfill
    \begin{minipage}[c]{0.68\linewidth}
        \centering
        \includegraphics[width=\linewidth, height=0.277\textheight, keepaspectratio]{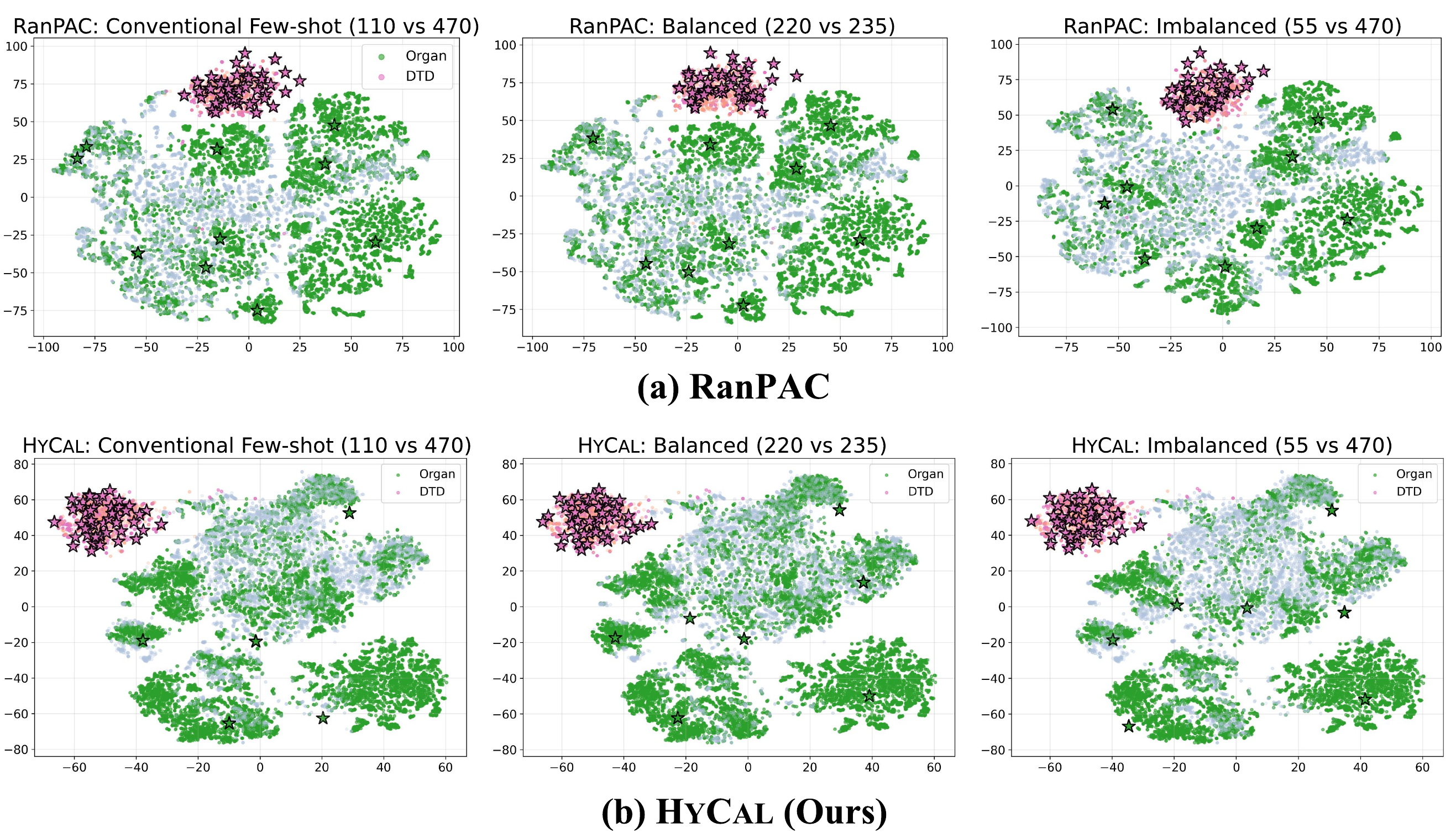}
        \captionof{figure}{t-SNE visualization showing prototype embeddings. 
Prototypes for $D^{med}$ are denoted by \includegraphics[height=0.8em]{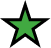}, 
and prototypes for $D^{tex}$ are denoted by \includegraphics[height=0.8em]{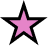}.}
    \label{fig:tsne-hycal_ranpac}
    \end{minipage}
\end{figure*}

In summary, our contributions are:
\begin{itemize}
\item We introduce Domain Gravity, a structural bias that explains representational drift under heterogeneous domain imbalance.
\item We present \textsc{HyCal}, a training-free cosine–Mahalanobis fusion method that mitigates Domain Gravity without modifying the backbone.
\item We define XD-VSCIL, a cross-discipline few-shot continual learning setting reflecting severe domain heterogeneity and imbalance.
\item We show through extensive experiments that \textsc{HyCal} consistently improves retention–adaptation performance and remains robust even under strong Domain Gravity.
\end{itemize}

\section{Related work}
\label{sec:relatedwork}

Research in cross-domain FSCIL increasingly leverages the strong feature representations of pretrained VLMs to support adaptation across heterogeneous visual domain~\cite{xu2024advancing,yu2024boosting,luo2025lada}. However, existing FSCIL and cross-domain adaptation methods struggle under domain heterogeneity. Many techniques focus on class imbalance or sample scarcity but implicitly assume homogeneous feature distributions and stable covariance estimation~\cite{goswami2023fecam,goswami2024calibrating,hong2024dynamically}. When domains differ in scale, variance, or orientation, metric calibration becomes biased and covariance regularization becomes unreliable in few-shot settings, failing to counteract the representational asymmetry characteristic of Domain Gravity.

To reduce domain discrepancy, several approaches introduce backbone-level adaptation modules or learnable prompts to align features across domains~\cite{zhou2022learning,li2023blip}. Other methods attempt to enrich or expand the representation space itself~\cite{feng2025defying,goswami2024calibrating,momeni2025continual}. RanPAC~\cite{feng2025defying} leverages random projections to obtain more expressive prototypes, while kernel-based feature expansions~\cite{goswami2024calibrating} aim to increase discriminative capacity without modifying the backbone. Although these strategies improve alignment to observed domains, they often amplify representational drift toward visually coherent or data-rich domains—precisely the effect captured by Domain Gravity—and generalize poorly to unseen or structurally dissimilar domains~\cite{CLIP,gao2024clip}. Prototype-based methods avoid backbone updates, but their reliance on stable covariance estimation or expanded feature spaces makes them fragile in heterogeneous few-shot regimes where Domain Gravity is strongest~\cite{goswami2023fecam,feng2025defying}. We focus on addressing the asymmetric representational forces induced by heterogeneous domains—a challenge that remains underexplored in training-free, few-shot scenarios.
\section{Cross-discipline variable few-shot class incremental learning (XD-VSCIL)}
\label{sec:xd-vscil}
\subsection{Preliminary analysis of Domain Gravity}

\begin{figure*}[ht]
\centering
\includegraphics[width=\textwidth]{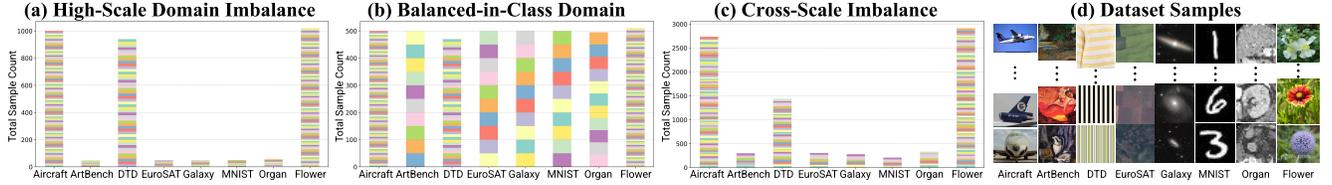}
\caption{
Overview of three representative dataset distributions used in XD-VSCIL.
(a)–(c): training-sample distributions under (a) highly domain-imbalanced setting, (b) balanced setting, and (c) class- and domain-wise imbalanced settings.
(d): example images from eight heterogeneous domains.
}
\label{fig:exp_setting}
\end{figure*}
We conduct a preliminary analysis to investigate how imbalanced domain priors affect the quality of representation in cross-domain FSCIL. We consider two heterogeneous domains:
$D^{med}$: OrganMNIST~\cite{OrganMNIST2}, a medical imaging dataset with 11 classes, and
$D^{tex}$: DTD~\cite{DTD}, a texture dataset with 47 classes.
These domains differ substantially in semantic structure and class cardinality. Experimental results are in \cref{tab:freeze_train_gap}.

In the standard few-shot CIL setting (10 shots per class), $D^{med}$ contains only 110 training samples while $D^{tex}$ contains 470, resulting in approximately 4× imbalance. In the \textit{variable-shot balanced} scenario, we adjust shots to equalize total samples: 20-shots for $D^{med}$ (220 samples) and 5-shots for $D^{tex}$ (235 samples), achieving nearly 1:1 ratio. Under this setting, $D^{med}$ improves by nearly 5\%p, highlighting the adverse effect of imbalanced domain priors. In the highly imbalanced scenario (5-shots for $D^{med}$, 10-shots for $D^{tex}$), yielding 1:9 ratio, performance collapses across multiple methods. Severe imbalance limits generalization, amplifies domain-specific challenges, and disrupts knowledge accumulation, reducing average performance by 2.42\%p.

To analyze representational behavior under imbalance, we visualize the embedding space via t-SNE. As shown in \cref{fig:tsne-hycal_ranpac}, prototype embeddings for RanPAC~\cite{mcdonnell2023ranpac} drift from corresponding domain clusters, especially for underrepresented domains. This reflects domain-dependent feature space deformation, where decision boundaries blur and inter-class entanglement increases. Despite using random projections to enrich representations, RanPAC fails to capture domain distributions. Prototypes remain biased toward dominant feature space regions, indicating persistent gravitational bias under strong Domain Gravity conditions.

These analyses demonstrate that imbalanced domain priors distort embedding geometry and existing methods, including projection-based approaches, struggle to maintain class separability. \textsc{HyCal} mitigates these issues by recalibrating prototypes using both directional and distributional cues, achieving more reliable representation alignment across heterogeneous domains.

\subsection{XD-VSCIL problem setting}

In XD-VSCIL, a model incrementally learns new classes from minimal examples originating from diverse domains, which are distinctly different from previously encountered data, while preserving knowledge of previously learned classes.
In the XD-VSCIL setting, learning progresses over discrete steps $t=1, \ldots, \mathcal{T}$. At each step $t$, the model learns from a new domain $\mathcal{D}^t = \{(x_i^t, y_i^t)\}_{i=1}^{K^t}$ corresponding to a novel academic discipline domain.
The model learns a class set $\mathcal{C}^t$ at each step $t$ under variable few-shot conditions, with limited per-class training samples $K_c^t$ per new class $c$, resulting in a total number of training samples $K^t = \sum_{c \in \mathcal{C}^t} K_c^t$. Novel classes $\mathcal{C}^t$ within domain $\mathcal{D}^t$ are disjoint from all previously encountered classes $\mathcal{C}_{seen}=\bigcup_{n=1}^{t-1}\mathcal{C}^n$.
Reflecting realistic data availability, $K_c$ values are not always fixed. Therefore, this setting naturally embodies an imbalanced data structure, with considerable variability in both the number of new class counts $|\mathcal{C}^t|$ and the number of per-class samples $K_c^t$ across different steps. As shown in \cref{fig:exp_setting} (a)-(c), this flexibility allows XD-VSCIL to be configured for various practical scenarios.
This cross-discipline approach implies substantial distribution shifts, as domain $\mathcal{D}^t$ significantly differs from previously encountered domains $\mathcal{D}^1, \ldots, \mathcal{D}^{t-1}$. 

\subsection{XD-VSCIL benchmark}

To illustrate the challenging cross-disciplinary and data-scarce nature of the XD-VSCIL setting, we utilize a sequence of eight benchmark datasets, as shown in \cref{fig:exp_setting} (d).
These datasets are intentionally selected for their diversity, covering distinct academic disciplines and visual domains: Aircraft (aeronautics)~\cite{Aircraft}, ArtBench (art)~\cite{ArtBench-10}, DTD (textures)~\cite{DTD}, EuroSAT (remote sensing)~\cite{EuroSAT}, Galaxy (astronomy)~\cite{Galaxy10}, MNIST (fundamental visual perception)~\cite{MNIST}, OrganMNIST (medical imaging)~\cite{OrganMNIST2}, and OxfordFlowers (fine-grained biology)~\cite{Flowers102}.
This selection ensures significant domain shifts between sequential tasks, embodying the cross-discipline aspect. 
Each task, representing a new dataset and discipline, adheres to the few-shot principle by providing only a limited number of samples per class, with variability in both class count and per-class sample availability to reflect the inherent data imbalance.

\subsection{XD-VSCIL metric}

In addition to the widely used average accuracy across all classes (i.e., Average accuracy) and Last accuracy~\cite{zhou2024class}, we propose a novel metric, Cross-Discipline Efficiency (CDE) score $\text{S}_\text{CDE}$. It is specifically designed for holistically evaluating zero-shot generalization, plasticity, forgetting mitigation, and data efficiency. 

To prioritize data efficiency, both adaptability component $\text{S}_\text{adapt}$ and Last accuracy component $\text{S}_\text{last}$ utilize weights $w^t \propto 1/\sqrt{K^t}$, where $K^t$ is the training data size for domain $\mathcal{D}^t$.

$\text{S}_\text{adapt}$ captures the model's ability to leverage existing knowledge and acquire new domain knowledge, where $Z^t$ is zero-shot accuracy and $A_t^t$ is the accuracy on domain $D^t$ after step $t$. $\text{S}_\text{last}$ is the weighted average of accuracies across all domains at the final step:
\begin{equation}
    \text{S}_\text{adapt} = \sum_{t=1}^{\mathcal{T}} \left(w^t \cdot \frac{Z^t +A_t^t}{2}\right), \text{S}_\text{last} = \sum_{t=1}^{\mathcal{T}} (w^t \cdot A_t^{\mathcal{T}})
\end{equation}

The proposed $\text{S}_\text{CDE}$ robustly captures model capability under the demanding XD-VSCIL setting. The overall $\text{S}_\text{CDE}$ is computed as:
\begin{equation}
    \text{S}_\text{CDE}=\frac{2 \cdot \text{S}_\text{adapt} \cdot \text{S}_\text{last}}{\text{S}_\text{adapt} + \text{S}_\text{last}},
    \label{eq:cilsds}
\end{equation}

where $\mathcal{S}_{CDE}=0$ if the denominator equals zero.

\section{Method}
\label{sec:method}
\begin{figure*}[ht]
\centering
\includegraphics[width=0.95\textwidth]{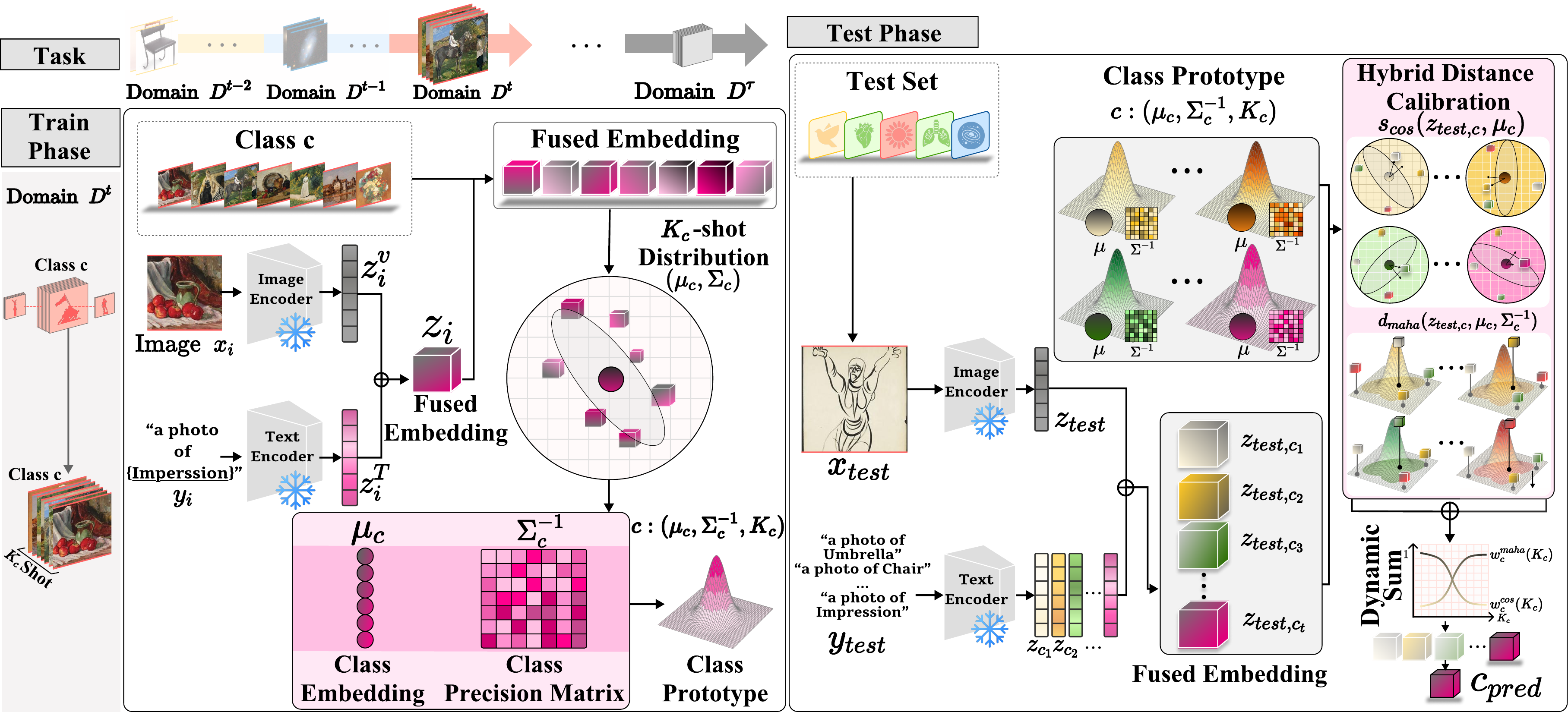}
\caption{
Overview of \textsc{HyCal}, our training-free prototype calibration method. For each XD-VSCIL task, class prototypes are constructed from frozen CLIP embeddings using a few sample images during training. At inference, \textsc{HyCal} combines cosine similarity and Mahalanobis distance to compute test-to-prototype scores, enabling robust classification under domain shift and data imbalance.}
\label{fig:overview}
\end{figure*}

We introduce \textsc{HyCal}, a training-free calibration method that leverages both cosine similarity and Mahalanobis distance to effectively match test embeddings with class prototypes, as illustrated in \cref{fig:overview}. This hybrid-distance formulation captures complementary directional and distributional cues—cosine similarity provides stable global orientation under heterogeneous domain drift, while Mahalanobis distance corrects domain-specific covariance and entropy differences.
\textsc{HyCal} operates through simple, explainable distance computations—requiring no updates, no additional parameters, and no model expansion—while providing stable prototype matching under heterogeneous few-shot conditions.

\subsection{Complementary roles of cosine and Mahalanobis measures}
Cosine similarity and Mahalanobis distance capture different geometric properties in high-dimensional spaces. Cosine similarity measures directional alignment, while Mahalanobis distance accounts for magnitude with covariance weighting. Combining them provides richer information than using either alone, which is the basis of \textsc{HyCal}.

For feature vectors $X, Y \in \mathbb{R}^d$, we define:
\begin{align}
C = \frac{X^\top Y}{\|X\| \|Y\|}, \quad M= \sqrt{(X - Y)^\top \Sigma^{-1} (X - Y)}.
\end{align}

To understand why these measures are complementary, consider the difference $\Delta = X - Y$ following an isotropic Gaussian $\Delta\sim \mathcal{N}(0, \sigma^2 I_d)$. We can decompose $\Delta = RU$, where $R = \|\Delta\| \sim \sigma \cdot \chi_d$ represents magnitude and $U = \Delta / \|\Delta\| \in \mathbb{S}^{d-1}$ represents direction.

Under this distribution, $R$ and $U$ are statistically independent ($R \perp U$). Since cosine similarity depends on $U$ and Mahalanobis distance depends on $R$ (and covariance), they capture independent, complementary information. This independence motivates combining both measures in \textsc{HyCal}.

\begin{table*}[t]
\caption{Performance under High-Scale Domain Imbalanced.
Average and Last accuracy (\%) across 8 domains, comparing parametric and parameter-free methods.
Results are averaged over 4 seeds with 95\% confidence intervals in subscripts. Domain names are abbreviated.}
\label{tab:high_scale}
\centering
\setlength{\tabcolsep}{0.9mm}
\resizebox{0.98\textwidth}{!}{%
\begin{tabular}{l c c c c c c c c c c}
\toprule\midrule\rowcolor{gray!24}
\multicolumn{11}{c}{\textbf{High-Scale Domain Imbalance}}
\\\midrule
& \rotatebox{0}{Aircraft}
& \rotatebox{0}{ArtBench}
& \rotatebox{0}{DTD}
& \rotatebox{0}{EuroSAT}
& \rotatebox{0}{Galaxy}
& \rotatebox{0}{MNIST}
& \rotatebox{0}{Organ}
& \rotatebox{0}{Flower}
& \rotatebox{0}{\textbf{Average}}
& \rotatebox{0}{{$\boldsymbol{\sigma}$}} \\ \toprule

\multicolumn{1}{l}{\quad Zero-shot}
  & 23.91 & 50.88 & 41.90 & 37.58 & 9.80 & 44.01 & 17.97 & 67.40 & 56.41 & 18.76
\\\midrule\hline

\multicolumn{11}{l}{\textbf{Average Acc.}}\\

\quad Primal-RAIL~\cite{xu2024advancing}
  & \textcolor{gray}{$37.33_{\pm 0.17}$} & \textcolor{gray}{$60.39_{\pm 0.74}$} & \textcolor{gray}{$68.79_{\pm 1.20}$ }& \textcolor{gray}{$72.40_{\pm 1.11}$}
  & \textcolor{gray}{$23.11_{\pm 8.42}$} & \textcolor{gray}{$77.77_{\pm 6.00}$} & \textcolor{gray}{$53.50_{\pm 0.75}$} & \textcolor{gray}{$90.30_{\pm 0.67}$}
  & \textcolor{gray}{$53.49_{\pm 0.79}$} & \textcolor{gray}{$22.04_{\pm 1.59}$} \\\hdashline

\quad FeCAM~\cite{goswami2023fecam}
  & $10.75_{\pm 3.36}$ & $0.00_{\pm 0.00}$  & $33.64_{\pm 12.40}$ & $0.00_{\pm 0.00}$
  & $0.00_{\pm 0.00}$  & $0.00_{\pm 0.00}$  & $0.00_{\pm 0.00}$   & $1.38_{\pm 2.61}$
  & $8.91_{\pm 2.75}$  & $11.91_{\pm 4.23}$ \\

\quad RanPAC~\cite{mcdonnell2023ranpac}
  & $\underline{36.99_{\pm 1.19}}$ & $\underline{36.34_{\pm 3.74}}$ & $66.35_{\pm 0.79}$ & $73.02_{\pm 5.11}$
  & $\mathbf{40.97_{\pm 6.77}}$ & $\mathbf{82.33_{\pm 4.10}}$ & $\mathbf{60.66_{\pm 4.46}}$
  & $92.82_{\pm 1.23}$ & $\underline{49.98_{\pm 0.61}}$ & $21.57_{\pm 1.61}$ \\

\quad KLDA~\cite{momeni2025continual}
  & $29.19_{\pm 0.47}$ & $31.66_{\pm 2.65}$ & $\underline{68.07_{\pm 1.66}}$ & $\mathbf{75.21_{\pm 4.56}}$
  & $36.85_{\pm 3.87}$ & $\underline{79.95_{\pm 3.97}}$ & $59.52_{\pm 3.92}$ & \textit{$95.19_{\pm 1.46}$}
  & $41.06_{\pm 0.80}$ & $24.61_{\pm 0.91}$ \\

\quad \textsc{HyCal}(Ours)
  & $\mathbf{40.97_{\pm 0.69}}$ & $\mathbf{49.99_{\pm 2.27}}$ & $\mathbf{69.84_{\pm 1.07}}$
  & $\underline{74.70_{\pm 1.71}}$ & $\underline{39.65_{\pm 2.62}}$ & $77.56_{\pm 6.05}$
  & $\underline{60.47_{\pm 4.63}}$ & $\mathbf{95.34_{\pm 0.78}}$
  & $\mathbf{54.48_{\pm 0.78}}$ & $19.50_{\pm 0.76}$ \\\midrule\hline

\multicolumn{11}{l}{\textbf{Last Acc.}}\\

\quad Primal-RAIL~\cite{xu2024advancing}
  & \textcolor{gray}{$36.91_{\pm 0.61}$ }& \textcolor{gray}{$60.44_{\pm 0.90}$} & \textcolor{gray}{$67.75_{\pm 0.68}$ }& \textcolor{gray}{$71.66_{\pm 1.11}$}
  & \textcolor{gray}{$20.90_{\pm 7.20}$} & \textcolor{gray}{$77.73_{\pm 6.05}$ }& \textcolor{gray}{$53.23_{\pm 0.61}$} & \textcolor{gray}{$90.30_{\pm 0.67}$}
  & \textcolor{gray}{$59.86_{\pm 1.51}$ }& \textcolor{gray}{$22.53_{\pm 1.36}$ }\\\hdashline

\quad FeCAM~\cite{goswami2023fecam}
  & $10.56_{\pm 3.43}$ & $0.00_{\pm 0.00}$ & $33.64_{\pm 12.40}$ & $0.00_{\pm 0.00}$
  & $0.00_{\pm 0.00}$  & $0.00_{\pm 0.00}$ & $0.00_{\pm 0.00}$   & $1.38_{\pm 2.61}$
  & $5.69_{\pm 1.90}$  & $11.90_{\pm 4.24}$ \\

\quad RanPAC~\cite{mcdonnell2023ranpac}
  & $37.20_{\pm 1.33}$ & $36.86_{\pm 2.51}$ & $66.71_{\pm 0.65}$ & $72.89_{\pm 4.74}$
  & $\mathbf{41.04_{\pm 6.97}}$ & $\mathbf{81.25_{\pm 4.61}}$
  & $\underline{60.25_{\pm 4.51}}$ & $92.82_{\pm 1.23}$
  & $61.13_{\pm 1.29}$ & $21.29_{\pm 1.67}$ \\

\quad KLDA~\cite{momeni2025continual}
  & $\underline{39.38_{\pm 0.57}}$ & $\underline{37.67_{\pm 2.96}}$ & $\underline{68.99_{\pm 1.64}}$
  & $\mathbf{75.18_{\pm 4.39}}$ & $37.11_{\pm 3.07}$ & $\underline{78.77_{\pm 4.20}}$
  & $59.17_{\pm 4.10}$ & $\underline{95.19_{\pm 1.46}}$
  & $\underline{61.43_{\pm 0.94}}$ & $21.89_{\pm 0.78}$ \\

\quad \textsc{HyCal}(Ours)
  & $\mathbf{40.98_{\pm 0.61}}$ & $\mathbf{49.68_{\pm 2.23}}$ & $\mathbf{69.59_{\pm 0.93}}$
  & $\underline{74.63_{\pm 1.78}}$ & $\underline{39.54_{\pm 2.62}}$ & $77.71_{\pm 5.24}$
  & $\mathbf{60.51_{\pm 4.73}}$ & $\mathbf{95.34_{\pm 0.78}}$
  & $\mathbf{63.50_{\pm 0.73}}$ & $19.55_{\pm 0.80}$ \\
\bottomrule

\end{tabular}
}
\end{table*}

\begin{figure*}[ht]
\centering
\includegraphics[width=0.98\textwidth]{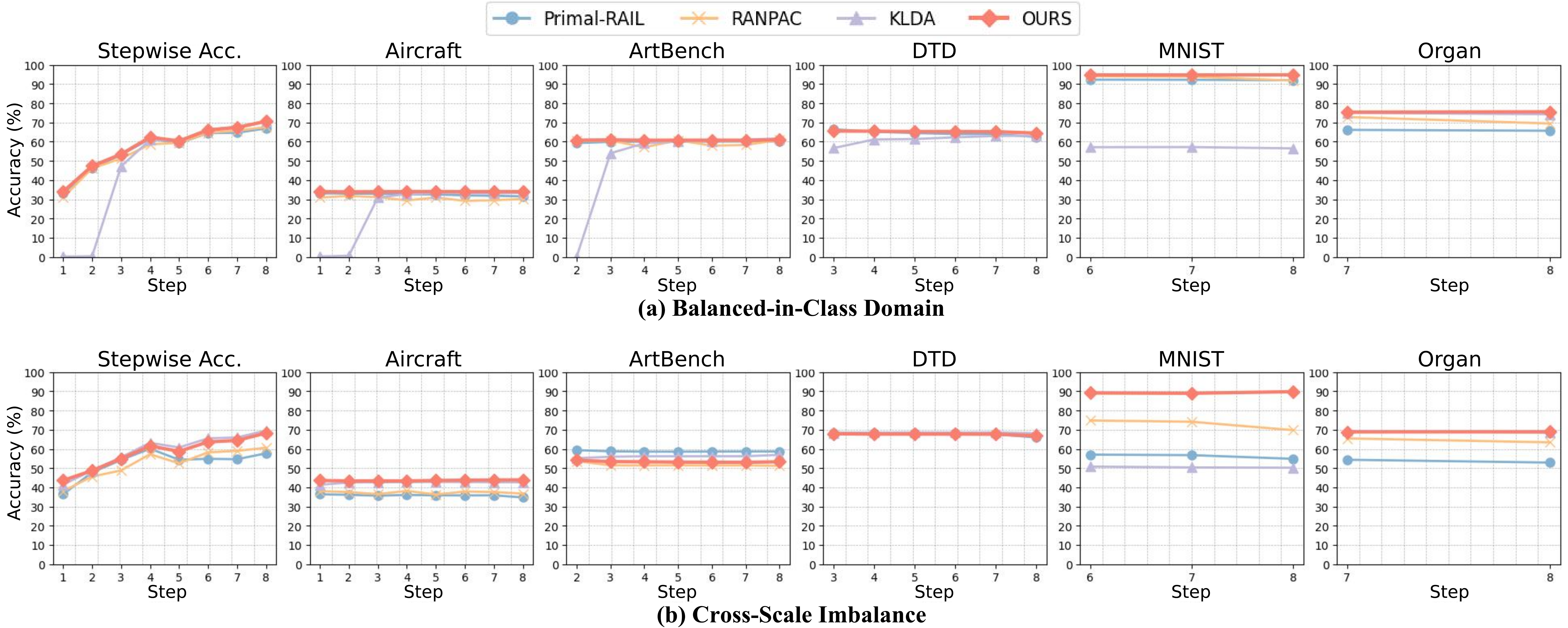} 
\caption{
Task-per accuracy across incremental steps under the (a) Balanced-in-class domain and (b) Cross-scale imbalance settings.
}
\label{fig:bal_cross}
\end{figure*}
\begin{theorem}[Independence of cosine similarity and Mahalanobis distance]
\label{thm:independence}
\textit{Cosine similarity depends only on direction, $C = f(U)$, and Mahalanobis distance only on norm, $M = g(R)$. Since $R \perp U$, it follows that $C \perp M$, and thus:}
\begin{align}
    H(C, M) = H(C) + H(M)
\end{align}
\end{theorem}
\begin{proof} 
Please see Supp.~\cref{sec:complementary_role}.
\end{proof}

This shows that the two measures provide non-overlapping information under isotropic conditions. Note that isotropy is used only for analytical clarity—not as a requirement for complementarity. While the isotropic Gaussian setting offers clean analysis~\cite{mardia2024multivariate,cai2013distributions}, the complementarity extends beyond this idealized case.

\subsection{Mutual information interpretation}
Let $L \in \{0, 1, \dots, n\}$ denote a class label. Define the mutual information quantities:
\begin{align}
    I(L; C), \quad I(L; M), \quad I(L; C, M)
\end{align}

\begin{theorem}[Information Gain from Combination]
\label{thm:mutual-information}
\textit{The mutual information between label $L$ and the combined similarity vector $(C, M)$ satisfies:}
\begin{align}
    I(L; C, M) &= I(L; C) + I(L; M \mid C) \\
    &\ge \max\{ I(L; C), I(L; M) \}
\end{align}
\end{theorem}
\begin{proof} 
Please see Supp.~\cref{sec:mutual}.
\end{proof}

This shows that combining cosine similarity and Mahalanobis distance yields strictly more discriminative information than either measure alone.

In practice, pretrained VLM embeddings are highly anisotropic, with varying feature norms and domain-dependent covariance structures. Under domain shifts, directional cues (captured by cosine) and covariance-aware cues (captured by Mahalanobis) diverge further~\cite{papadimitriou2025interpreting, qi2025beyond}, reinforcing their complementarity beyond the isotropic case. Please see Supp.~\cref{sec:strong_complementarity,sec:connection} for further discussion and diagnostics.

Therefore, \textsc{HyCal} exploits this complementarity to extract richer geometric information from frozen embeddings, enabling robust prototype matching without fine-tuning—even under strong domain heterogeneity.

\subsection{Incremental prototype learning}
We use a pre-trained CLIP model with frozen image encoder $f_V(\cdot)$ and text encoder $f_T(\cdot)$ throughout all incremental learning stages. This preserves CLIP's rich general-purpose features for cross-domain generalization while enabling efficient updates without costly backbone retraining.

When a new domain $\mathcal{D}^t$ arrives with novel classes $\mathcal{C}^t$ and only $K_c^t$ samples per class $c \in \mathcal{C}^t$, we incrementally learn and store a prototype for each new class. This strictly follows the CIL paradigm by restricting access to data from previous domains $\mathcal{D}^n$ where $n < t$.

At step $t$, given domain $\mathcal{D}^t$ with classes $\mathcal{C}^t = \{c_j\}_{j=1}^{L^t}$ where $L^t = |\mathcal{C}^t|$, we construct textual input $y_j$ using the template ``a photo of a $c_j$'' for each class name $c_j$.

For image $x_i$, we obtain visual embedding $z_{i}^{V} = f_V(x_i)$, and for text $y_i$, we obtain $z_{i}^{T} = f_T(y_i)$. Given $K_c^t$ training images per class $c$, we generate a fused embedding $z_i$ for each pair $(x_i, y_i)$ by:
\begin{equation}
z_i = [z_{i}^{V}; z_{i}^T],
\end{equation}
where $[;]$ denotes sum-based fusion. All prototype learning occurs in this unified space.

For each class $c \in \mathcal{C}^t$ with $K_c^t$ training samples $\{(x_i^c, y_i^c)\}_{i=1}^{K_c^t}$, we compute fused embeddings $\{z_i^c\}_{i=1}^{K_c^t}$. The class prototype $\mathcal{P}_c$ consists of mean embedding ${\mu}_c$, regularized precision matrix ${\Sigma}_c^{-1}$, and sample size $K_c^t$.

The mean embedding is:
\begin{equation}
    {\mu}_c = \frac{1}{K_{c}^t} \sum_{i=1}^{K_c^t} {z}_i^c
    \label{eq:mean_embedding}
\end{equation}

For the precision matrix, we first estimate empirical covariance ${\Sigma}_c$ and apply regularization for robustness in few-shot settings:
\begin{equation}
    \Sigma_c^{reg} = (1-\lambda)\mathbf{\Sigma}_c + \lambda \gamma \mathbf{I},
    \label{eq:reg_covariance}
\end{equation}
where $\lambda \in [0,1]$ is the regularization parameter, $\mathbf{I}$ is the identity matrix, and $\gamma$ is a small constant based on global feature variance. The precision matrix is ${\Sigma}_c^{-1} = ({\Sigma}_c^{reg})^{-1}$.
The set of learned prototypes up to step $t$ is $\mathcal{P}_{seen}^t = \{\mathcal{P}_c \mid c \in \mathcal{C}_{seen}^t\}$.

\subsection{Inference with hybrid distance prototype calibration}
\label{ssec:classification}

During inference after step $t$, given a test image $x_{test}$ and all previously seen classes $\mathcal{C}_{seen}^t$, we perform classification using \textsc{HyCal}. For each candidate class $c \in \mathcal{C}_{seen}^t$, we construct a class-conditional query embedding ${z}_{test,c}$ by fusing the test image embedding $f_V(x_{test})$ with the text embedding $f_T(y_c)$ for class $c$.

We compute two complementary metrics between ${z}_{test,c}$ and prototype $\mathcal{P}_c = \{{\mu}_c, {\Sigma}_c^{-1}, K_c^t\}$:

\noindent\textbf{Mahalanobis distance ($d_{maha}$).} This captures intra-class variability using inverse class covariance~\cite{ren2021simple}:
\begin{equation}
    d_{maha}(z_{test,c}, {\mu}_c, \Sigma_c^{-1}) = ({z}_{test,c} - {\mu}_c)^T \Sigma_c^{-1} ({z}_{test,c} - {\mu}_c)
    \label{eq:mahalanobis}
\end{equation}
For fusion, $d_{maha}$ is normalized to $[0, 1]$.

\noindent\textbf{Cosine similarity ($s_{cos}$).} This measures angular similarity to prototype mean ${\mu}_c$:
\begin{equation}
    s_{cos}({z}_{test,c}, {\mu}_c) = \frac{{z}_{test,c} \cdot {\mu}_c}{\|{z}_{test,c}\| \|{\mu}_c\|}
    \label{eq:cosine_similarity}
\end{equation}

\noindent\textbf{Dynamic summation.} The final prediction $c_{pred}$ for $x_{test}$ is obtained by adaptively combining $d_{maha}$ and $s_{cos}$ with weight $w_c$:
\begin{align}
    c_{pred} &= \arg\max_{c \in \mathcal{C}_{seen}^t}\left[w_c \cdot d_{maha} + (1-w_c) \cdot s_{cos}\right]\nonumber\\
    &(where, w_c = \frac{1}{1+\exp\big(-(K_c^{t}-\alpha)/\beta\big)})
    \label{eq:final_classification}
\end{align}
$\alpha$ is a threshold and $\beta$ is a scaling hyperparameter. This decision rule leverages both directional and distributional information for robust classification in incremental learning scenarios.

\section{Experiments}

\noindent \textbf{Experimental setting.} All experiments use frozen CLIP ViT-B/16~\cite{CLIP}, and compared methods training-free approaches FeCAM~\cite{goswami2023fecam}, RanPAC~\cite{mcdonnell2023ranpac} and KLDA~\cite{momeni2025continual}, as well as the trainable baseline Primal-RAIL~\cite{xu2024advancing}.
Each result is averaged over 4 runs with random seeds 0, 1, 42, and 1993 to ensure robustness. Further implementation details and the hyperparameter are provided in Supp.~\cref{sec:experimental_setting}.

\begin{table}[t!]
\caption{Performance on $\text{S}_\text{CDE}$ across XD-VSCIL settings.}
\label{tab:CDE}
\centering
\setlength{\tabcolsep}{0.4mm}
\scriptsize
\resizebox{0.8\columnwidth}{!}{%
\begin{tabular}{lccc}
\toprule
Method   & $\boldsymbol{\text{S}_{\text{adapt}}}$  & $\boldsymbol{\text{S}_{\text{last}}}$ & $\boldsymbol{\text{S}_{\text{CDE}}}$ \\
\hline\midrule\rowcolor{gray!8}
\multicolumn{4}{c}{\textbf{Balanced-in-Class Domain}} \\ 
\midrule
Primal-RAIL~\cite{xu2024advancing}
  & \textcolor{gray}{$52.77_{\pm 0.34}$}  & \textcolor{gray}{$67.39_{\pm 0.68}$}  & \textcolor{gray}{$59.19_{\pm 0.47}$} \\\hdashline
FeCAM~\cite{goswami2023fecam}
  & $10.80_{\pm 14.71}$ & $11.55_{\pm 1.95}$  & $10.06_{\pm 6.65}$ \\
RanPAC~\cite{mcdonnell2023ranpac}
  & $\mathbf{56.42_{\pm 11.30}}$ & $63.97_{\pm 12.04}$ & $\underline{59.29_{\pm 0.36}}$ \\
KLDA~\cite{momeni2025continual}
  & $47.47_{\pm 0.25}$  & $\mathbf{70.61_{\pm 0.90}}$ & $56.77_{\pm 0.47}$ \\
\textbf{\textsc{HyCal} (Ours)}
  & $\underline{53.63_{\pm 0.14}}$ & $\underline{70.49_{\pm 0.23}}$ & $\mathbf{60.92_{\pm 0.17}}$ \\
\midrule \midrule
\rowcolor{gray!16}
\multicolumn{4}{c}{\textbf{Cross-Scale Imbalance}} \\ 
\midrule
Primal-RAIL~\cite{xu2024advancing}
  & \textcolor{gray}{$46.95_{\pm 0.83}$}  & \textcolor{gray}{$57.27_{\pm 1.67}$}  & \textcolor{gray}{$51.60_{\pm 1.18}$} \\\hdashline
FeCAM~\cite{goswami2023fecam}
  & $10.67_{\pm 3.58}$   & $20.05_{\pm 7.06}$  & $13.91_{\pm 4.68}$ \\
RanPAC~\cite{mcdonnell2023ranpac}
  & $49.32_{\pm 2.56}$   & $62.08_{\pm 5.98}$  & $54.96_{\pm 3.92}$ \\
KLDA~\cite{momeni2025continual}
  & $\mathbf{52.25_{\pm 0.48}}$ & $\mathbf{69.39_{\pm 1.15}}$ & $\mathbf{59.61_{\pm 0.72}}$ \\
\textbf{\textsc{HyCal} (Ours)}
  & $\underline{51.99_{\pm 0.35}}$ & $\underline{68.83_{\pm 0.57}}$ & $\underline{59.24_{\pm 0.41}}$ \\
\midrule \midrule
\rowcolor{gray!24}
\multicolumn{4}{c}{\textbf{High-Scale Domain Imbalance}} \\ 
\midrule
Primal-RAIL~\cite{xu2024advancing}
  & \textcolor{gray}{$49.41_{\pm 10.41}$} & \textcolor{gray}{$57.86_{\pm 2.22}$} & \textcolor{gray}{$53.18_{\pm 6.51}$} \\\hdashline
FeCAM~\cite{goswami2023fecam}
  & $5.15_{\pm 13.52}$   & $1.85_{\pm 0.62}$  & $1.84_{\pm 2.08}$ \\
RanPAC~\cite{mcdonnell2023ranpac}
  & $\mathbf{50.07_{\pm 8.59}}$ & $56.44_{\pm 10.35}$ & $\underline{52.57_{\pm 1.16}}$ \\
KLDA~\cite{momeni2025continual}
  & $42.39_{\pm 0.64}$   & $\underline{58.79_{\pm 1.47}}$ & $49.26_{\pm 0.95}$ \\
\textbf{\textsc{HyCal} (Ours)}
  & $\underline{47.71_{\pm 0.53}}$ & $\mathbf{61.39_{\pm 1.05}}$ & $\mathbf{53.70_{\pm 0.74}}$ \\
\bottomrule
\end{tabular}
}
\end{table}

\begin{table*}[ht]
\caption{Performance across 5, 10, 15, 20-shots on FSCIL benchmarks, averaged over 4 seeds with 95\% confidence intervals.}
\label{tab:multishot}
\centering
\normalsize
\setlength{\tabcolsep}{1pt}
\renewcommand{\arraystretch}{1.02}
\resizebox{\textwidth}{!}{%

\begin{tabular}{l ccc ccc ccc ccc}
\toprule
            & \multicolumn{3}{c}{\textbf{5-shot}}
            & \multicolumn{3}{c}{\textbf{10-shot}}
            & \multicolumn{3}{c}{\textbf{15-shot}}
            & \multicolumn{3}{c}{\textbf{20-shot}} \\ 
            & Last Acc. & Avg Acc. & $\boldsymbol{\text{S}_{\text{CDE}}}$
            & Last Acc. & Avg Acc. & $\boldsymbol{\text{S}_{\text{CDE}}}$
            & Last Acc. & Avg Acc. & $\boldsymbol{\text{S}_{\text{CDE}}}$
            & Last Acc. & Avg Acc. & $\boldsymbol{\text{S}_{\text{CDE}}}$ \\
\midrule

Primal-RAIL~\cite{xu2024advancing}
 & \textcolor{gray}{$59.40_{\pm1.88}$} & \textcolor{gray}{$51.61_{\pm1.10}$} & \textcolor{gray}{$52.58_{\pm1.73}$}
 & \textcolor{gray}{$59.40_{\pm1.88}$} & \textcolor{gray}{$51.61_{\pm1.10}$} & \textcolor{gray}{$52.58_{\pm1.73}$}
 & \textcolor{gray}{$59.40_{\pm1.88}$} & \textcolor{gray}{$51.61_{\pm1.10}$} & \textcolor{gray}{$52.58_{\pm1.73}$}
 & \textcolor{gray}{$59.40_{\pm1.88}$} & \textcolor{gray}{$51.61_{\pm1.10}$} & \textcolor{gray}{$52.58_{\pm1.73}$} \\\hdashline

FeCAM~\cite{goswami2023fecam}
 & $8.90_{\pm1.12}$ & $8.48_{\pm0.28}$ & $4.27_{\pm6.13}$
 & $11.16_{\pm3.26}$ & $14.62_{\pm4.23}$ & $7.22_{\pm1.76}$
 & $22.48_{\pm6.51}$ & $25.29_{\pm2.84}$ & $21.57_{\pm8.40}$
 & $21.37_{\pm3.18}$ & $25.06_{\pm2.13}$ & $18.31_{\pm10.96}$ \\

RanPAC~\cite{mcdonnell2023ranpac}
 & $\underline{58.04_{\pm1.53}}$ & $\underline{45.48_{\pm2.89}}$ & $\underline{51.94_{\pm1.10}}$
 & $64.34_{\pm0.26}$ & $\underline{53.13_{\pm0.78}}$ & $\underline{56.19_{\pm0.11}}$
 & $66.52_{\pm0.41}$ & $56.62_{\pm0.40}$ & $57.86_{\pm0.30}$
 & $67.85_{\pm0.24}$ & $58.29_{\pm0.40}$ & $\underline{58.70_{\pm0.17}}$ \\

KLDA~\cite{momeni2025continual}
 & $56.54_{\pm0.61}$ & $7.46_{\pm0.08}$ & $29.13_{\pm0.06}$
 & $\underline{66.01_{\pm0.46}}$ & $43.62_{\pm0.44}$ & $53.34_{\pm0.34}$
 & $\underline{68.55_{\pm0.51}}$ & $\underline{57.19_{\pm0.63}}$ & $\underline{58.78_{\pm0.39}}$
 & $\underline{69.73_{\pm0.88}}$ & $\underline{59.65_{\pm0.61}}$ & $52.13_{\pm0.68}$ \\

\textbf{\textsc{HyCal} (Ours)}
 & $\mathbf{60.82_{\pm0.97}}$ & $\mathbf{50.58_{\pm0.74}}$ & $\mathbf{60.36_{\pm1.35}}$
 & $\mathbf{66.52_{\pm0.28}}$ & $\mathbf{56.69_{\pm0.47}}$ & $\mathbf{57.27_{\pm0.21}}$
 & $\mathbf{69.02_{\pm0.37}}$ & $\mathbf{59.27_{\pm0.62}}$ & $\mathbf{59.06_{\pm0.22}}$
 & $\mathbf{70.67_{\pm0.21}}$ & $\mathbf{60.90_{\pm0.56}}$ & $\mathbf{60.21_{\pm0.06}}$ \\
\bottomrule
\end{tabular}
}
\end{table*}

\begin{table}[ht]
\centering
\caption{Performance of individual distance metrics and their combinations across varying imbalance settings.}
\setlength{\tabcolsep}{0.4mm}
\scriptsize
\resizebox{0.8\columnwidth}{!}{%
\begin{tabular}{lccc}
\toprule
Method   & Last Acc. & Avg Acc. & $\boldsymbol{\text{S}_{\text{CDE}}}$ \\
\midrule
\rowcolor{gray!8}
\multicolumn{4}{c}{\textbf{Balanced-in-Class Domain}} \\ 
\midrule
Cosine similarity        & $64.67_{\pm0.82}$ & $54.00_{\pm0.82}$ & $56.82_{\pm0.55}$ \\
Mahalanobis Distance     & $70.34_{\pm0.28}$ & $57.42_{\pm0.27}$ & $60.78_{\pm0.19}$ \\
Average                  & $69.80_{\pm0.41}$ & $57.48_{\pm0.64}$ & $60.32_{\pm0.17}$ \\
\textbf{Dynamic (Ours)} & $\mathbf{70.54_{\pm0.23}}$ & $\mathbf{57.75_{\pm0.45}}$ & $\mathbf{60.92_{\pm0.17}}$ \\
\midrule
\rowcolor{gray!16}
\multicolumn{4}{c}{\textbf{Cross-Scale Imbalance}} \\
\midrule
Cosine similarity        & $65.09_{\pm0.64}$ & $55.85_{\pm0.62}$ & $56.44_{\pm0.59}$ \\
Mahalanobis Distance     & $65.85_{\pm1.47}$ & $55.68_{\pm1.46}$ & $56.99_{\pm1.41}$ \\
Average                  & $67.63_{\pm1.39}$ & $57.97_{\pm1.36}$ & $58.24_{\pm1.16}$ \\
\textbf{Dynamic (Ours)} & $\mathbf{68.69_{\pm0.44}}$ & $\mathbf{58.43_{\pm0.51}}$ & $\mathbf{59.24_{\pm0.41}}$ \\
\midrule
\rowcolor{gray!24}
\multicolumn{4}{c}{\textbf{High-Scale Domain Imbalance}} \\
\midrule
Cosine similarity        & $60.09_{\pm1.65}$ & $51.46_{\pm1.44}$ & $51.29_{\pm1.50}$ \\
Mahalanobis Distance     & $63.10_{\pm0.85}$ & $53.68_{\pm0.71}$ & $53.50_{\pm0.87}$ \\
Average                  & $63.33_{\pm0.81}$ & $54.35_{\pm0.95}$ & $53.68_{\pm0.76}$ \\
\textbf{Dynamic (Ours)} & $\mathbf{63.50_{\pm0.73}}$ & $\mathbf{54.48_{\pm0.78}}$ & $\mathbf{53.70_{\pm0.74}}$ \\
\bottomrule
\end{tabular}
}
\label{tab:dismetric}
\end{table}

\subsection{XD-VSCIL}

We evaluate models under three XD-VSCIL settings that capture imbalance conditions of varying severity. Our primary focus is the \textit{High-scale domain imbalance} setting, which reflects a realistic and challenging scenario where domains differ substantially in total sample size while maintaining uniform shots per class, thus inducing pronounced domain-level skew that exposes vulnerabilities in existing methods. Additionally, we examine the \textit{Balanced-in-class domain} setting as a controlled baseline where domains contribute comparable data and classes receive identical shots, and the \textit{Cross-scale imbalance} setting where each class receives 5–50 samples, representing variation in both class-wise and domain-wise sample scales that bridges the balanced baseline and high-scale imbalance scenarios. The learning order follows alphabetical order: Aircraft, ArtBench, DTD, EuroSAT, Galaxy, MNIST, OrganMNIST, and OxfordFlowers.

\cref{tab:high_scale} presents the results under the High-scale domain imbalance setting. 
Our proposed \textsc{HyCal} demonstrates superior performance in highly challenging scenario, achieving substantial improvements of over 4.5\%p in Average accuracy compared to second-best training-free method. 
\cref{fig:bal_cross} illustrates the task-per accuracy across incremental steps for the Balanced-in-class domain and Cross-scale imbalance settings. As shown in the first column, existing methods exhibit slower accuracy growth under the imbalanced setting compared to the balanced counterpart when domain data scales are imbalanced. Moreover, particularly evident in the MNIST results, prior methods suffer from degraded accuracy in the imbalanced setting relative to the balanced baseline. In contrast, our proposed \textsc{HyCal} maintains robust performance with minimal degradation across both settings, demonstrating its resilience against Domain Gravity throughout the incremental learning process. As shown in~\cref{tab:CDE}, \textsc{HyCal} also achieves the best $\text{S}_\text{CDE}$ across all settings, further validating its stable performance and balanced knowledge acquisition. Full numerical results for the Balanced-in-class and Cross-scale imbalance settings are provided in Supp.~\cref{sec:result_bal_cross}.

\subsection{FSCIL}

We also evaluate \textsc{HyCal} under standard FSCIL settings. As shown in~\cref{tab:multishot}, \textsc{HyCal} consistently outperforms existing methods across all shot settings, including low-shot regimes. While prior approaches~\cite{momeni2025continual} exhibit substantial performance degradation when base data is limited, \textsc{HyCal} efficiently extracts and utilizes information even when data is extremely scarce, making it well-suited for challenging few-shot environments.

\subsection{Ablation study}
\noindent \textbf{Distance metric.}
As shown in \cref{tab:dismetric}, cosine similarity and Mahalanobis distance exhibit complementary strengths. Under balanced conditions, Mahalanobis distance outperforms cosine similarity. As imbalance intensifies, their gap narrows, yet combining both through weighted summation consistently outperforms either metric alone. In the setting, Cross-scale imbalance, our method achieves $59.24$ $\text{S}_\text{CDE}$, surpassing both cosine and Mahalanobis, demonstrating synergistic improvements. By dynamically adjusting contributions based on $K_c^t$, \textsc{HyCal} adapts flexibly to varying domain conditions.

\noindent \textbf{Robustness to domain order.}
To evaluate robustness to varying domain sequences, we conduct experiments under random domain orders. As shown in the Supp.~\cref{sec:robust_order}, \textsc{HyCal} consistently strongest or near-strongest across all four random orderings, demonstrating order-invariance. This stems from our hybrid calibration mechanism, which adapts to each domain independently rather than relying on fixed geometric assumptions, effectively mitigating Domain Gravity regardless of domain arrival sequence.

\section{Conclusion}
We introduce \textit{Domain Gravity}, a structural bias arising when pretrained VLM embeddings are adapted across heterogeneous visual disciplines with imbalanced data. This representational asymmetry poses a fundamental barrier for few-shot cross-discipline continual learning. To address this, we established \textbf{XD-VSCIL}, capturing real-world domain heterogeneity and variable supervision.
We propose \textbf{\textsc{HyCal}}, a training-free method combining cosine similarity and Mahalanobis distance to leverage complementary geometric cues—directional alignment and covariance-aware magnitude. Our adaptive weighted summation dynamically adjusts each metric's contribution, enabling flexible adaptation to domain imbalance.
Extensive experiments demonstrate that \textsc{HyCal} delivers consistently superior retention–adaptation performance, remains remarkably stable under severe domain asymmetry, and effortlessly scales to continually expanding heterogeneous data—all without fine-tuning or any additional parameters.

\section*{Acknowledgments}
This work was supported by the Institute of Information \& Communications Technology Planning \& Evaluation (IITP) grant funded by the Korea government (MSIT) [RS-2021-II211341, Artificial Intelligence Graduate School Program (Chung-Ang University)] and by the National Research Foundation of Korea (NRF) grant funded by the Korea government (MSIT) (RS-2025-00556246).
{
    \small
    \bibliographystyle{ieeenat_fullname}
    \bibliography{main}
}

\clearpage
\setcounter{page}{1}

\maketitlesupplementary

\renewcommand{\thetheorem}{\arabic{theorem}}
\setcounter{theorem}{0}

\section{Supplementary on the complementary roles and mutual information between cosine and Mahalanobis measures}
\label{sec:suppl_proof}

This supplementary section provides additional analysis supporting the complementary relationship between cosine similarity and Mahalanobis distance, as formalized in~\cref{thm:independence} and~\cref{thm:mutual-information_suppl}. While both measures are computed from the same embedding pairs, they capture fundamentally different geometric and statistical aspects—angular alignment versus covariance-adjusted magnitude—which together yield richer and more discriminative information than either measure alone.

\subsection{Complementary roles of cosine and Mahalanobis measures}
\label{sec:complementary_role}

\subsubsection{Preliminaries and statistical independence}
Let $X, Y \in \mathbb{R}^d$ be random vectors and define their difference:
\begin{align*}
    \Delta = X - Y.
\end{align*}

We consider the following similarity measures:
\begin{align*}
    C &= \frac{X^\top Y}{\|X\| \|Y\|}, \\
    M &= \sqrt{(X - Y)^\top \Sigma^{-1} (X - Y)}.
\end{align*}

Cosine similarity $C$ captures angular alignment between $X$ and $Y$, while Mahalanobis distance $M$ reflects the covariance-adjusted magnitude of their difference. Although both are derived from the same pair $(X, Y)$, they emphasize distinct geometric aspects.

To obtain a tractable theoretical perspective, we adopt an idealized isotropic surrogate model. Assume that
\[
    \Delta \sim \mathcal{N}(0, \sigma^2 I_d),
\]
and that Mahalanobis distance is computed with the true covariance $\Sigma = \sigma^2 I_d$. In this case,
\[
    M = \sqrt{\Delta^\top (\sigma^2 I_d)^{-1} \Delta}
      = \frac{\|\Delta\|}{\sigma}.
\]
Moreover, we can represent $\Delta$ in polar form as
\[
    \Delta = R U,
\]
where
\[
    R = \|\Delta\| \sim \sigma \cdot \chi_d, \qquad
    U = \frac{\Delta}{\|\Delta\|} \in \mathbb{S}^{d-1}.
\]
For an isotropic Gaussian, the magnitude $R$ and the direction $U$ are statistically independent:
\[
    R \perp U.
\]

In our surrogate model, we further idealize cosine similarity and Mahalanobis distance as being driven by these independent radial and angular factors: Mahalanobis distance depends only on $R$, while cosine similarity depends only on an angular component (which we denote generically by $U$). Formally, we model
\[
    C = f(U), \qquad M = g(R),
\]
for some measurable functions $f$ and $g$. This does not claim that empirical cosine similarity is exactly a function of $\Delta/\|\Delta\|$, but instead captures the intuition that cosine similarity is dominantly governed by angular variation. In contrast, the Mahalanobis distance is primarily determined by radial variation.

\begin{theorem}[Independence of Cosine Similarity and Mahalanobis Distance under an Isotropic Model]
Under the isotropic surrogate model described above, cosine similarity depends only on the angular factor $U$ and Mahalanobis distance depends only on the radial factor $R$:
\[
    C = f(U), \qquad M = g(R).
\]
Since $R \perp U$, it follows that $C \perp M$, and thus their joint entropy decomposes as
\begin{align*}
    H(C, M) = H(C) + H(M).
\end{align*}
\end{theorem}

\begin{proof}
For an isotropic Gaussian $\Delta \sim \mathcal{N}(0, \sigma^2 I_d)$, the density of $\Delta \in \mathbb{R}^d$ is rotationally invariant and depends only on $\|\Delta\|$. In hyperspherical coordinates $\Delta = r u$ with $r = \|\Delta\|$ and $u \in \mathbb{S}^{d-1}$, the joint density factorizes as
\begin{align*}
    p(r, u) = c \, r^{d-1} \exp\left(-\frac{r^2}{2\sigma^2}\right) p(u),
\end{align*}
where $c$ is a normalization constant and $p(u)$ is the uniform density on the unit sphere. Hence, $R \perp U$.

By our surrogate model assumption, $C = f(U)$ and $M = g(R)$ for measurable functions $f$ and $g$. Since measurable functions of independent random variables remain independent, we obtain $C \perp M$. Therefore, their joint entropy factorizes:
\begin{align*}
    H(C, M)
    &= - \int p(c, m) \log p(c, m) \, dc \, dm \\
    &= - \int p(c) p(m) \log\bigl(p(c)p(m)\bigr) \, dc \, dm \\
    &= - \int p(c) \log p(c) \, dc
       - \int p(m) \log p(m) \, dm \\
    &= H(C) + H(M).
\end{align*}
\end{proof}

This result formalizes the idea that cosine similarity and Mahalanobis distance capture statistically independent aspects of the pair $(X, Y)$ under an isotropic surrogate model: one is driven by angular variation, and the other by radial variation. In particular, the pair $(C, M)$ encodes the sum of their individual entropies, reflecting complementary information about the underlying embedding geometry.

\subsection{Mutual information interpretation}
\label{sec:mutual}
Let $L \in \{0, 1, \dots, n\}$ denote a class label, and suppose that $L$ depends on the relative relationship between $C$ and $M$. We are interested in the mutual information quantities
\begin{align*}
    I(L; C), \quad I(L; M), \quad I(L; C, M).
\end{align*}

\begin{theorem}[Information Gain from Combination]
\label{thm:mutual-information_suppl}
The mutual information between the label $L$ and the combined similarity vector $(C, M)$ satisfies
\begin{align*}
    I(L; C, M)
    &= I(L; C) + I(L; M \mid C) \\
    &\ge \max\{ I(L; C), I(L; M) \}.
\end{align*}
\end{theorem}

\begin{proof}
By the chain rule for mutual information, we have
\begin{align*}
    I(L; C, M) = I(L; C) + I(L; M \mid C).
\end{align*}
By definition of conditional mutual information,
\begin{align*}
    I(L; M \mid C)
    &= \mathbb{E}_{c \sim p(C)} \Bigl[
        D_{\mathrm{KL}}\bigl(
            p(L,M \mid C=c) \,\|\,
        \\
    &\hspace{5.8em}
            p(L \mid C=c)\,p(M \mid C=c)
        \bigr)
    \Bigr].
\end{align*}
and since the Kullback--Leibler divergence satisfies $D_{\mathrm{KL}}(\cdot \,\|\, \cdot) \ge 0$, it follows that
\[
    I(L; M \mid C) \ge 0.
\]
Therefore,
\begin{align*}
    I(L; C, M)
    &= I(L; C) + I(L; M \mid C) \\
    &\ge I(L; C).
\end{align*}
By symmetry, $I(L; C, M) \ge I(L; M)$ as well, which yields
\begin{align*}
    I(L; C, M) \ge \max\{ I(L; C), I(L; M) \}.
\end{align*}
\end{proof}

Together with the independence result above, this theorem shows that cosine similarity and Mahalanobis distance are information-theoretically complementary: under the isotropic surrogate model, they encode independent angular and radial factors, and when labels depend on both, their combination $(C, M)$ provides at least as much mutual information with $L$ as either measure alone. This complementarity enables richer information extraction from frozen embeddings while keeping the backbone entirely fixed.

\subsection{Real VLM embeddings: Why complementarity becomes stronger}
\label{sec:strong_complementarity}
The independence result above is derived under an idealized isotropic Gaussian surrogate, in which the angular and radial components factorize. In real pretrained VLMs, however, the embedding geometry is inherently anisotropic: feature norms exhibit heavy-tailed distributions, covariance structures vary significantly across domains, and domain shifts induce direction–magnitude distortions in heterogeneous ways~\cite{papadimitriou2025interpreting, qi2025beyond}. 

These anisotropic effects break the exact statistical independence between cosine similarity and Mahalanobis distance. Yet, paradoxically, this makes the two measures \emph{more} complementary in practice. Under anisotropy, directional cues captured by cosine similarity and covariance-aware magnitude cues captured by Mahalanobis distance tend to decouple, encoding increasingly orthogonal information regarding the data distribution, especially under domain shifts~\cite{park2023understanding,levi2025double}. 

Therefore, the theoretical independence shown in the isotropic case should be viewed as a conceptual lower bound of complementarity. In real VLM embeddings, the geometric anisotropy amplifies the distinct information encoded by each measure, making their combination strictly more discriminative than either measure alone.

\begin{figure*}[t]
  \centering
  \includegraphics[width=0.90\linewidth]{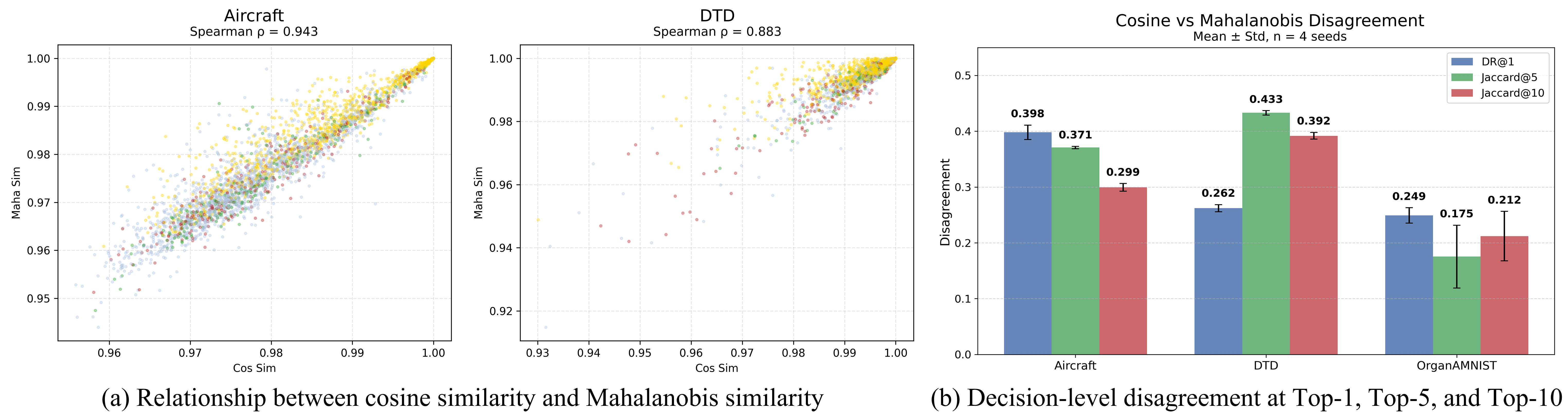}
   \caption{Cosine–Mahalanobis: (a) Relationship and (b) Ranking where cosine-correct (green) and Mahalanobis-correct (red) samples only partially overlap (yellow).}
   \label{fig:cos-maha}
\end{figure*}

\begin{figure*}[t]
  \centering
  \includegraphics[width=0.80\linewidth]{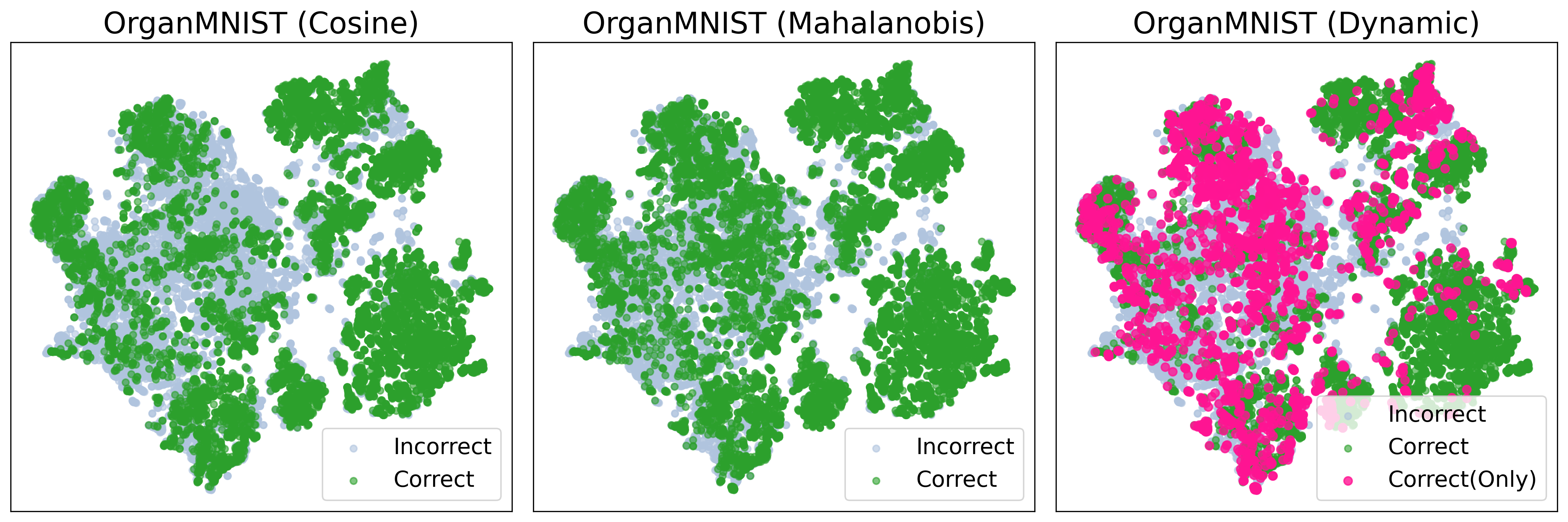}
   \caption{Qualitative comparison of distance metrics. Pink points show samples only correctly classified by dynamic summation.} 
   \label{fig:onecol}
\end{figure*}

\subsection{Theoretical--empirical connection in \textsc{HyCal}}
\label{sec:connection}

\cref{sec:complementary_role,sec:mutual,sec:strong_complementarity} establish that cosine similarity and Mahalanobis distance provide different but complementary views of prototype matching: cosine similarity captures directional alignment, whereas Mahalanobis distance measures deviation relative to class-specific covariance. \textsc{HyCal} translates this observation into a practical inference rule by combining both cues when evaluating each seen class. This is particularly useful in XD-VSCIL, where heterogeneous domains arrive sequentially and often with severe imbalance. Under such conditions, domain-level asymmetry can distort the representation space and make a single metric overly biased toward only one geometric aspect. By reusing frozen pretrained representations and calibrating decisions at the prototype level, \textsc{HyCal} reduces sensitivity to such domain-level dominance without requiring backbone adaptation or task-specific retraining.

The empirical results are consistent with this interpretation, while also clarifying the limits of the theory. Real CLIP embeddings are anisotropic, so our analysis is not intended as an exact statistical description of the feature space. 
Nevertheless, the diagnostics in \cref{fig:cos-maha,fig:onecol} show that cosine similarity and Mahalanobis distance make complementary decisions: they disagree on nontrivial subsets of samples and recover different correct predictions, especially in sparse domains such as Aircraft, DTD, and OrganMNIST. 
This complementarity translates into strong behavior under challenging imbalance settings, where \textsc{HyCal} remains robust on both overall metrics and low-data domains. 
\cref{tab:rebuttal-table} further shows that the same qualitative tendency extends across pretrained models. With \texttt{OpenCLIP ViT-B/16} pretrained on \texttt{DataComp-1B}~\cite{ilharco_gabriel_2021_5143773}, \textsc{HyCal} achieves the best performance across all three metrics, and under the X-TAIL~\cite{xu2024advancing} setting it attains the best Last Acc. while remaining competitive on the other metrics. Taken together, these results support the practical value of the theoretical complementarity developed in this section, while making clear that the theory serves primarily as intuition for real VLM embeddings rather than a strict independence claim.

\begin{table}[t]
\caption{Performance under High-scale domain imbalance.}
\label{tab:rebuttal-table}
\centering
\renewcommand{\arraystretch}{1.2}
\setlength{\tabcolsep}{1mm}
\resizebox{0.99\linewidth}{!}{%
\begin{tabular}{lcccccc}
\hline
             & \multicolumn{3}{c}{\textbf{PTM: \texttt{DataComp-1B}~\cite{ilharco_gabriel_2021_5143773}}}   & \multicolumn{3}{c}{\textbf{Setting: X-TAIL~\cite{xu2024advancing}}}     \\
Method       & Last Acc.      & Avg Acc.       & $\text{S}_\text{CDE}$     & Last Acc.      & Avg Acc.       & $\text{S}_\text{CDE}$     \\ \hline
Primal-RAIL~\cite{xu2024advancing}  & \textcolor{gray}{65.89}          & \textcolor{gray}{58.10}    & \textcolor{gray}{56.40}    & \textcolor{gray}{73.82}          & \textcolor{gray}{66.79}          & \textcolor{gray}{69.48}          \\\hdashline
RanPAC~\cite{mcdonnell2023ranpac}   & {\ul 66.20}    & {\ul 55.78}          & {\ul 55.64}          & 69.74          & {\ul 67.60}    & 68.20          \\
KLDA~\cite{momeni2025continual}   & 63.95          & 41.24          & 51.50          & {\ul 76.71}          & \textbf{71.38} & \textbf{72.42} \\
\textbf{\textsc{HyCal} (Ours)}& \textbf{68.88} & \textbf{59.89} & \textbf{57.94} & \textbf{77.44} & 64.41          & {\ul 70.66}    \\ \hline
\end{tabular}}
\end{table}

\section{Additional ablation study}

To further evaluate the stability and design choices of \textsc{HyCal}, we conduct ablation studies on three components: robustness to domain order, sensitivity to hyperparameters, and the effect of different image–text embedding fusion strategies. These analyses assess whether the method maintains consistent performance under different task permutations, whether its calibration behavior is sensitive to the selection of $(\alpha,\beta)$, and whether our sum-based image–text fusion scheme provides advantages over the commonly used concatenation strategy.

\begin{figure}[t]
\centering
\includegraphics[width=0.99\linewidth]{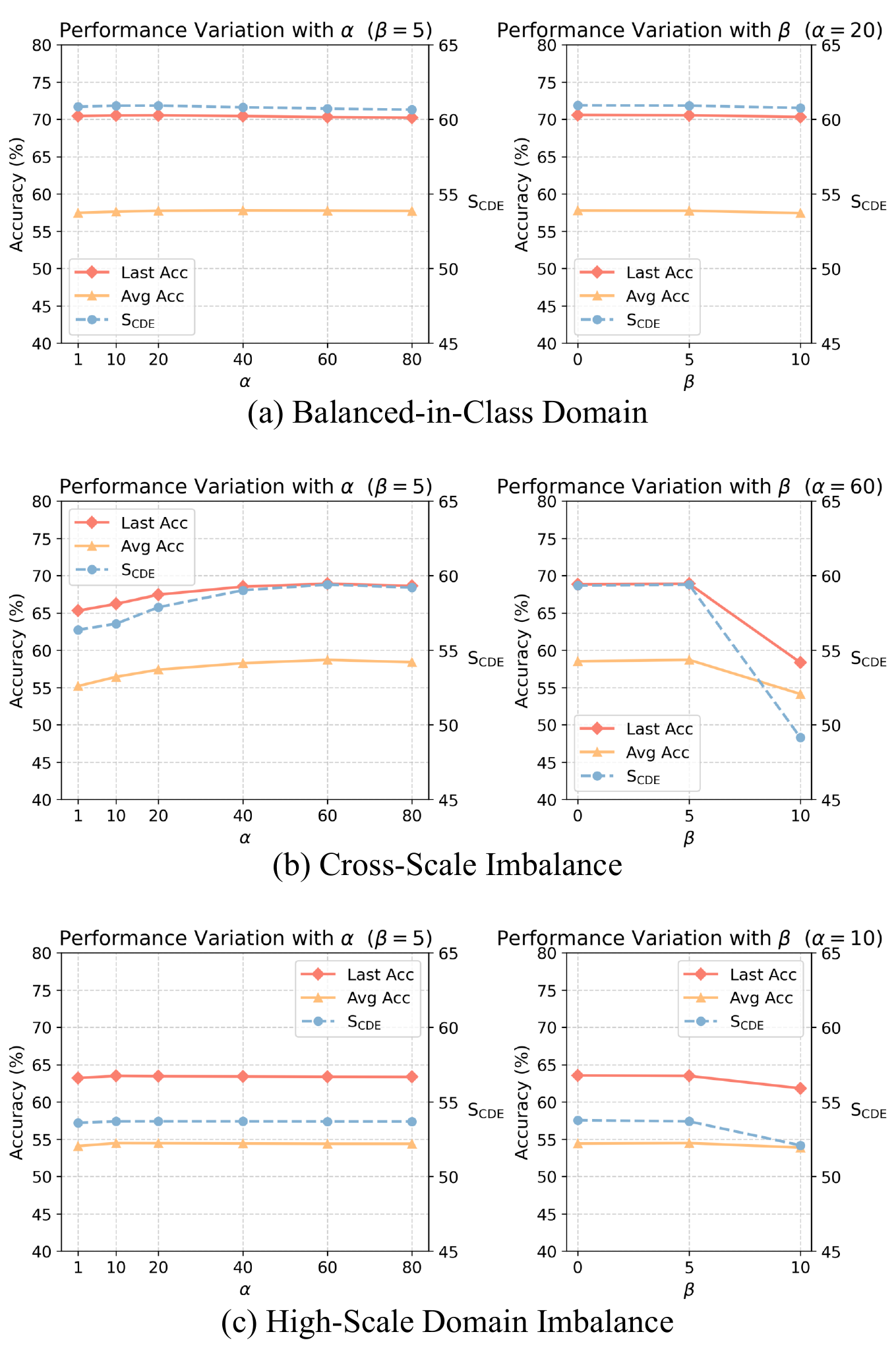} 
\caption{Hyperparameter sensitivity analysis for $\alpha$ and $\beta$. Each curve shows the performance variation when sweeping one parameter while keeping the other fixed.}
\label{fig:hyperparmeter}
\end{figure}

\begin{figure}[t]
\centering
\includegraphics[width=0.98\linewidth]{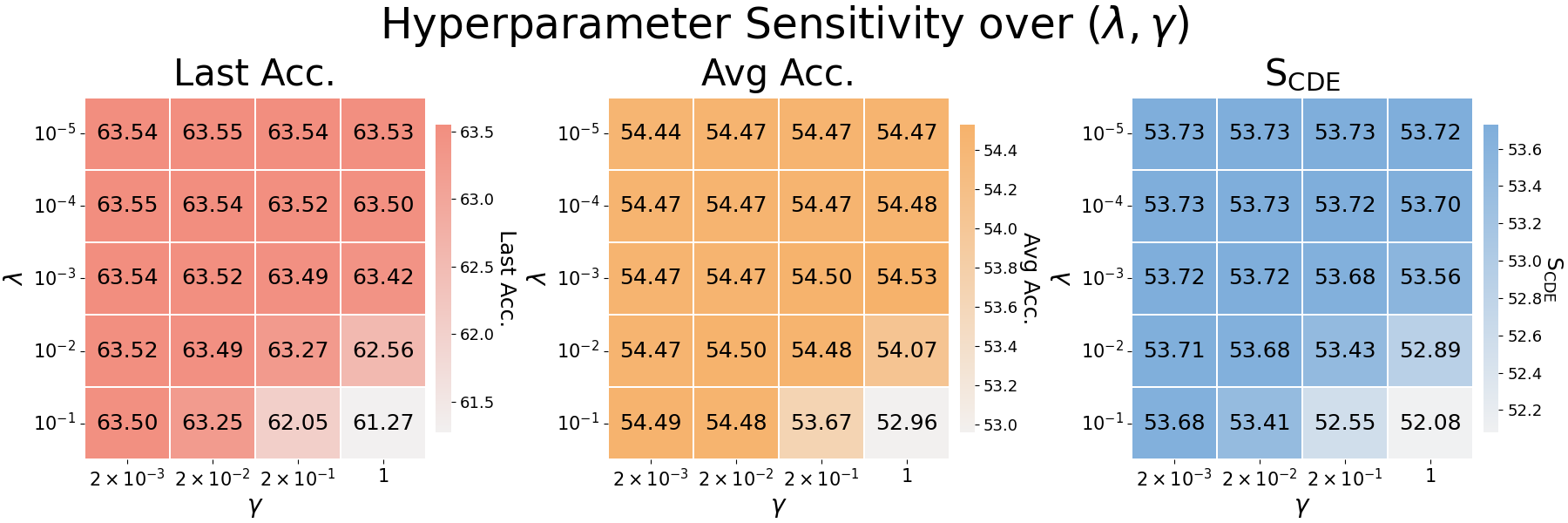}
\caption{Hyperparameter sensitivity analysis for $\lambda$ and $\gamma$ under the High-scale domain imbalance setting. Heatmaps show Last Acc., Avg Acc., and $\text{S}_{\text{CDE}}$ over different $(\lambda, \gamma)$ configurations.}
\label{fig:hyperparmeter_lambda_gamma}
\end{figure}

\subsection{Robustness to domain order}
\label{sec:robust_order}
To examine the robustness of each method under variations in domain sequences, we evaluate all approaches across four random domain orders, derived from fixed random seeds \(\{0,1,42,1993\}\). 
Each seed generates a unique permutation of the eight domains---Aircraft, ArtBench, DTD, EuroSAT, Galaxy, MNIST, Organ, and Flower---which are presented sequentially as incremental tasks. These four permutations are referred to as Random order I–IV. 
\begin{itemize}
    \item \textit{Random order I (seed = 0)}\\
    Galaxy $\rightarrow$ ArtBench $\rightarrow$ MNIST $\rightarrow$ DTD $\rightarrow$ Aircraft $\rightarrow$ EuroSAT $\rightarrow$ Flower $\rightarrow$ Organ

    \item \textit{Random order II (seed = 1)}\\
    EuroSAT $\rightarrow$ Organ $\rightarrow$ ArtBench $\rightarrow$ MNIST $\rightarrow$ Flower $\rightarrow$ Aircraft $\rightarrow$ Galaxy $\rightarrow$ DTD

    \item \textit{Random order III (seed = 42)}\\
    EuroSAT $\rightarrow$ Galaxy $\rightarrow$ Organ $\rightarrow$ Flower $\rightarrow$ DTD $\rightarrow$ MNIST $\rightarrow$ Aircraft $\rightarrow$ ArtBench

    \item \textit{Random order IV (seed = 1993)}\\
    ArtBench $\rightarrow$ MNIST $\rightarrow$ Organ $\rightarrow$ Aircraft $\rightarrow$ EuroSAT $\rightarrow$ DTD $\rightarrow$ Galaxy $\rightarrow$ Flower
\end{itemize}
The overall performance of all compared methods under Random order I–IV is summarized in \cref{tab:randorder}, covering three XD-VSCIL settings. 

Across all domain orders and settings, \textsc{HyCal} consistently achieves the highest performance, indicating that its calibration mechanism remains stable and effective regardless of task permutation. Furthermore, to provide a more fine-grained perspective, \cref{tab:randomorder_V2} reports the class-wise performance under the High-scale domain imbalanced setting for each of the four random domain orders.

\begin{table}[t!]
\centering
\caption{Comparison of fusion strategies under three domain conditions.}
\label{tab:fusion}
\scriptsize
\resizebox{0.98\columnwidth}{!}{%
\begin{tabular}{lccc}
\toprule
Fusion   & Last Acc. & Avg Acc. & $\boldsymbol{\text{S}_{\text{CDE}}}$ \\
\midrule
\midrule\rowcolor{gray!8}
\multicolumn{4}{c}{\textbf{Balanced-in-Class Domain}} \\ 
\midrule
Concat                 & $70.25_{\pm0.20}$ & $57.33_{\pm0.32}$ & $60.70_{\pm0.15}$ \\
\textbf{Sum (Ours)}    & $\mathbf{70.54_{\pm0.23}}$ & $\mathbf{57.75_{\pm0.45}}$ & $\mathbf{60.92_{\pm0.17}}$ \\
\midrule
\midrule\rowcolor{gray!16}
\multicolumn{4}{c}{\textbf{Cross-Scale Imbalance}} \\
\midrule

Concat                 & $68.30_{\pm1.02}$ & $57.96_{\pm1.16}$ & $58.95_{\pm0.79}$ \\
\textbf{Sum (Ours)}    & $\mathbf{68.69_{\pm0.44}}$ & $\mathbf{58.43_{\pm0.51}}$ & $\mathbf{59.24_{\pm0.41}}$ \\

\midrule
\midrule\rowcolor{gray!24}
\multicolumn{4}{c}{\textbf{High-Scale Domain Imbalance}} \\
\midrule
Concat                 & $63.06_{\pm0.63}$ & $53.94_{\pm0.76}$ & $53.39_{\pm0.71}$ \\
\textbf{Sum (Ours)}    & $\mathbf{63.50_{\pm0.73}}$ & $\mathbf{54.48_{\pm0.78}}$ & $\mathbf{53.70_{\pm0.74}}$ \\
\bottomrule
\end{tabular}
}
\end{table}

\begin{table*}[t]
\caption{Performance comparison across random-order domain setting. Results are averaged over 4 seeds with 95\% confidence intervals in subscripts.}
\label{tab:randorder}
\centering
\normalsize
\setlength{\tabcolsep}{1pt}
\renewcommand{\arraystretch}{1.1}
\resizebox{\textwidth}{!}{%
\begin{tabular}{l cccc cccc cccc ccc}
\toprule 
\midrule
 & \multicolumn{4}{c}{\textit{Random Order I}}
 & \multicolumn{4}{c}{\textit{Random Order II}}
 & \multicolumn{4}{c}{\textit{Random Order III}}
 & \multicolumn{3}{c}{\textit{Random Order IV}} \\ 
 & Last Acc. & Avg Acc. & $\text{S}_\text{CDE}$ &
 & Last Acc. & Avg Acc. & $\text{S}_\text{CDE}$ &
 & Last Acc. & Avg Acc. & $\text{S}_\text{CDE}$ &
 & Last Acc. & Avg Acc. & $\text{S}_\text{CDE}$ \\
\midrule\rowcolor{gray!8}
\multicolumn{16}{c}{\textbf{Balanced-in-Class Domain}} 
\\ \midrule

Primal-RAIL~\cite{xu2024advancing}
 & \textcolor{gray}{$67.90_{\pm0.59}$} & \textcolor{gray}{$63.53_{\pm0.85}$} & \textcolor{gray}{$59.77_{\pm0.40}$} &
 & \textcolor{gray}{$67.67_{\pm0.38}$}& \textcolor{gray}{$76.41_{\pm0.31}$} & \textcolor{gray}{$59.45_{\pm0.26}$} &
 & \textcolor{gray}{$67.63_{\pm0.87}$} & \textcolor{gray}{$74.26_{\pm1.00}$} & \textcolor{gray}{$59.42_{\pm0.51}$} &
 & \textcolor{gray}{$67.58_{\pm0.56}$} & \textcolor{gray}{$68.18_{\pm0.38}$} & \textcolor{gray}{$59.78_{\pm0.40}$} \\\hdashline

FeCAM~\cite{goswami2023fecam}
 & $15.26_{\pm2.99}$ & $18.17_{\pm2.30}$ & $10.16_{\pm1.99}$&
 & $12.63_{\pm2.07}$ & $21.95_{\pm6.73}$ & $10.24_{\pm1.20}$&
 & $13.97_{\pm4.85}$ & $20.81_{\pm5.63}$ & $9.30_{\pm3.23}$&
 & $11.70_{\pm5.90}$ & $27.09_{\pm7.01}$ & $12.90_{\pm4.08}$ \\

RanPAC~\cite{mcdonnell2023ranpac}
 & $67.60_{\pm0.25}$ & $\underline{63.43_{\pm0.13}}$ & $\underline{59.13_{\pm0.17}}$&
 & $67.21_{\pm1.00}$ & $\underline{75.84_{\pm1.27}}$ & $\underline{59.05_{\pm0.63}}$&
 & $66.03_{\pm3.00}$ & $\underline{74.16_{\pm0.18}}$ & $\underline{58.71_{\pm1.16}}$&
 & $67.30_{\pm0.42}$ & $\underline{66.81_{\pm0.60}}$ & $\underline{59.04_{\pm0.20}}$ \\

KLDA~\cite{momeni2025continual}
 & $\mathbf{70.80_{\pm0.30}}$ & $51.39_{\pm0.36}$ & $56.41_{\pm0.06}$&
 & $\underline{70.14_{\pm0.21}}$ & $58.11_{\pm0.47}$ & $53.47_{\pm0.28}$&
 & $\underline{70.56_{\pm0.59}}$ & $57.61_{\pm0.66}$ & $55.04_{\pm0.42}$&
 & $\underline{70.50_{\pm0.59}}$ & $53.95_{\pm0.27}$ & $54.42_{\pm0.37}$ \\

\textbf{\textsc{HyCal} (Ours)}
 & $\underline{70.58_{\pm0.32}}$ & $\mathbf{64.40_{\pm0.59}}$ & $\mathbf{60.91_{\pm0.23}}$&
 & $\mathbf{70.53_{\pm0.24}}$ & $\mathbf{78.53_{\pm0.33}}$ & $\mathbf{60.86_{\pm0.16}}$&
 & $\mathbf{70.69_{\pm0.38}}$ & $\mathbf{76.18_{\pm0.24}}$ & $\mathbf{61.07_{\pm0.27}}$&
 & $\mathbf{70.65_{\pm0.30}}$ & $\mathbf{70.24_{\pm0.35}}$ & $\mathbf{60.95_{\pm0.23}}$ \\

\midrule \midrule
\rowcolor{gray!16}
\multicolumn{16}{c}{\textbf{Cross-Scale Imbalance}} \\
\midrule

Primal-RAIL~\cite{xu2024advancing}
 & \textcolor{gray}{$58.40_{\pm1.54}$} & \textcolor{gray}{$53.57_{\pm3.12}$} & \textcolor{gray}{$52.76_{\pm1.65}$}&\quad
 & \textcolor{gray}{$58.31_{\pm1.18}$} & \textcolor{gray}{$67.40_{\pm1.81}$} & \textcolor{gray}{$52.74_{\pm1.25}$}&\quad
 & \textcolor{gray}{$58.09_{\pm1.16}$} & \textcolor{gray}{$64.70_{\pm1.94}$} & \textcolor{gray}{$52.05_{\pm1.49}$}&\quad
 & \textcolor{gray}{$58.14_{\pm0.68}$} & \textcolor{gray}{$59.18_{\pm0.56}$} & \textcolor{gray}{$52.94_{\pm0.92}$} \\\hdashline

FeCAM~\cite{goswami2023fecam}
 & $16.65_{\pm3.51}$ & $19.75_{\pm5.18}$ & $13.41_{\pm3.88}$&\quad
 & $17.56_{\pm6.27}$ & $22.46_{\pm10.36}$ & $13.57_{\pm4.80}$&\quad
 & $16.98_{\pm3.92}$ & $20.12_{\pm7.07}$ & $13.03_{\pm3.41}$&\quad
 & $17.31_{\pm2.68}$ & $21.23_{\pm6.35}$ & $13.85_{\pm2.13}$ \\

RanPAC~\cite{mcdonnell2023ranpac}
 & $61.77_{\pm6.12}$ & $\underline{56.97_{\pm4.13}}$ & $\underline{57.49_{\pm4.03}}$&\quad
 & $59.27_{\pm7.85}$ & $\underline{69.77_{\pm3.21}}$ & $\underline{55.11_{\pm4.44}}$&\quad
 & $59.39_{\pm7.56}$ & $\underline{68.39_{\pm3.33}}$ & $\underline{56.46_{\pm4.52}}$&\quad
 & $61.50_{\pm5.64}$ & $\underline{61.49_{\pm3.46}}$ & $\underline{54.98_{\pm2.97}}$ \\

KLDA~\cite{momeni2025continual}
 & $\mathbf{69.26_{\pm0.44}}$ & $43.53_{\pm0.74}$ & $48.14_{\pm0.90}$&\quad
 & $\underline{69.34_{\pm0.78}}$ & $46.79_{\pm0.61}$ & $47.68_{\pm0.36}$&\quad
 & $\mathbf{69.31_{\pm0.16}}$ & $47.67_{\pm0.38}$ & $48.29_{\pm0.39}$&\quad
 & $\mathbf{69.12_{\pm1.04}}$ & $44.27_{\pm0.55}$ & $46.76_{\pm1.18}$ \\

\textbf{\textsc{HyCal} (Ours)}
 & $\underline{68.69_{\pm0.96}}$ & $\mathbf{61.27_{\pm1.76}}$ & $\mathbf{58.99_{\pm0.71}}$&{\quad}
 & $\mathbf{68.60_{\pm0.58}}$ & $\mathbf{73.70_{\pm0.35}}$ & $\mathbf{58.99_{\pm0.46}}$&{\quad}
 & $\underline{68.64_{\pm0.85}}$ & $\mathbf{71.94_{\pm0.81}}$ & $\mathbf{59.06_{\pm0.76}}$&{\quad}
 & $\underline{68.64_{\pm1.39}}$ & $\mathbf{66.52_{\pm1.06}}$ & $\mathbf{58.97_{\pm1.08}}$ \\

\midrule \midrule
\rowcolor{gray!24}
\multicolumn{16}{c}{\textbf{High-Scale Domain Imbalance}} \\
\midrule

Primal-RAIL~\cite{xu2024advancing}
 & \textcolor{gray}{$59.80_{\pm1.18}$} & \textcolor{gray}{$52.60_{\pm3.58}$} & \textcolor{gray}{$52.18_{\pm1.23}$}&
 & \textcolor{gray}{$59.28_{\pm1.57}$} & \textcolor{gray}{$65.78_{\pm1.53}$} & \textcolor{gray}{$51.51_{\pm1.38}$}&
 & \textcolor{gray}{$59.62_{\pm0.55}$} & \textcolor{gray}{$61.48_{\pm1.27}$} & \textcolor{gray}{$52.10_{\pm0.45}$}&
 & \textcolor{gray}{$60.28_{\pm0.61}$} & \textcolor{gray}{$61.58_{\pm0.48}$} & \textcolor{gray}{$52.74_{\pm0.60}$} \\\hdashline

FeCAM~\cite{goswami2023fecam}
 & $5.68_{\pm0.73}$ & $8.93_{\pm0.74}$ & $2.53_{\pm0.21}$&
 & $5.53_{\pm1.31}$ & $7.90_{\pm0.25}$ & $2.70_{\pm2.70}$&
 & $5.66_{\pm0.40}$ & $8.64_{\pm3.29}$ & $2.35_{\pm0.56}$&
 & $6.00_{\pm1.11}$ & $10.15_{\pm1.25}$ & $2.94_{\pm0.57}$ \\

RanPAC~\cite{mcdonnell2023ranpac}
 & $59.80_{\pm1.68}$ & $\underline{51.61_{\pm1.85}}$ & $\mathbf{54.40_{\pm1.62}}$&
 & $59.79_{\pm1.38}$ & $\underline{65.28_{\pm1.53}}$ & $\underline{52.81_{\pm0.65}}$&
 & $59.77_{\pm1.51}$ & $\underline{63.91_{\pm1.21}}$ & $\mathbf{54.37_{\pm1.12}}$&
 & $60.31_{\pm0.63}$ & $\underline{56.08_{\pm2.20}}$ & $\underline{51.24_{\pm0.91}}$ \\

KLDA~\cite{momeni2025continual}
 & $\underline{61.68_{\pm1.24}}$ & $31.03_{\pm0.81}$ & $41.85_{\pm0.94}$&
 & $\underline{60.67_{\pm1.13}}$ & $32.89_{\pm0.74}$ & $34.25_{\pm1.08}$&
 & $\underline{61.13_{\pm0.52}}$ & $39.60_{\pm1.12}$ & $40.81_{\pm0.28}$&
 & $\underline{60.99_{\pm0.68}}$ & $31.44_{\pm1.53}$ & $39.28_{\pm0.51}$ \\

\textbf{\textsc{HyCal} (Ours)}
 & $\mathbf{63.39_{\pm0.48}}$ & $\mathbf{54.96_{\pm0.81}}$ & $\underline{53.47_{\pm0.37}}$&
 & $\mathbf{63.01_{\pm1.08}}$ & $\mathbf{65.71_{\pm1.17}}$ & $\mathbf{53.13_{\pm0.94}}$&
 & $\mathbf{63.15_{\pm0.74}}$ & $\mathbf{64.73_{\pm1.70}}$ & $\underline{53.15_{\pm0.73}}$&
 & $\mathbf{63.03_{\pm1.11}}$ & $\mathbf{59.69_{\pm1.70}}$ & $\mathbf{53.06_{\pm1.22}}$ \\

\bottomrule
\end{tabular}
}

\end{table*}

\begin{table*}[t]
\centering
\caption{Performance across random domain orders (I–IV) under the High-scale domain imbalanced setting. Results are averaged over 4 seeds with 95\% confidence intervals in subscripts. Domain names are abbreviated.}
\vspace{-3mm}
\label{tab:randomorder_V2}
\setlength{\tabcolsep}{1mm}
\renewcommand{\arraystretch}{1.0}
\resizebox{0.98\textwidth}{!}{%
\begin{tabular}{l c c c c c c c c c c c}
\toprule\midrule
\rowcolor{gray!24}
\multicolumn{12}{c}{\textit{Random Order I}} 
\\\midrule
& \rotatebox{0}{Galaxy}
& \rotatebox{0}{ArtBench}
& \rotatebox{0}{MNIST}
& \rotatebox{0}{DTD}
& \rotatebox{0}{Aircraft}
& \rotatebox{0}{EuroSAT}
& \rotatebox{0}{Flower}
& \rotatebox{0}{Organ} &{\quad}
& \rotatebox{0}{\textbf{Average}}
& \rotatebox{0}{{$\boldsymbol{\sigma}$}} \\ \toprule

\multicolumn{1}{l}{\quad Zero-shot}
  & 9.80 & 50.88 & 44.01 & 41.90 & 23.91 & 37.58 & 67.40 & 17.97&{\quad} & 56.41 & 18.76
\\\midrule\hline

\multicolumn{11}{l}{\textbf{Average Acc.}}\\

\quad Primal-RAIL~\cite{xu2024advancing}
  & \textcolor{gray}{$25.37_{\pm 10.57}$}
  & \textcolor{gray}{$60.23_{\pm 0.72}$}
  & \textcolor{gray}{$76.99_{\pm 6.49}$}
  & \textcolor{gray}{$68.93_{\pm 1.35}$}
  & \textcolor{gray}{$37.45_{\pm 1.18}$}
  & \textcolor{gray}{$72.42_{\pm 1.14}$}
  & \textcolor{gray}{$90.49_{\pm 0.58}$}
  & \textcolor{gray}{$53.83_{\pm 3.88}$}&{\quad}
  & \textcolor{gray}{$52.60_{\pm 3.58}$}
  & \textcolor{gray}{$21.47_{\pm 2.62}$}
\\\hdashline

\quad FeCAM~\cite{goswami2023fecam}
  & $2.89_{\pm 1.46}$ & $5.76_{\pm 0.64}$ & $0.00_{\pm 0.00}$ & $31.77_{\pm 5.68}$
  & $11.88_{\pm 0.98}$ & $0.00_{\pm 0.00}$ & $1.76_{\pm 3.04}$ & $0.00_{\pm 0.00}$&{\quad}
  & $8.93_{\pm 0.74}$  & $10.92_{\pm 1.86}$ \\

\quad RanPAC~\cite{mcdonnell2023ranpac}
  & $\underline{37.02_{\pm 5.25}}$ & $\underline{40.30_{\pm 5.96}}$ & $\mathbf{77.63_{\pm 3.51}}$ & $\underline{66.33_{\pm 0.55}}$
  & $35.95_{\pm 1.11}$ & $71.66_{\pm 4.17}$ & $92.45_{\pm 0.48}$ & $56.07_{\pm 3.90}$&{\quad}
  & $\underline{51.61_{\pm 1.85}}$ & $20.96_{\pm 0.72}$ \\

\quad KLDA~\cite{momeni2025continual}
  & $19.98_{\pm 0.50}$ & $22.89_{\pm 2.32}$ & $53.28_{\pm 3.68}$ & $54.83_{\pm 1.18}$
  & $\underline{39.53_{\pm 1.95}}$ & $\mathbf{76.60_{\pm 2.75}}$ & $\mathbf{95.74_{\pm 0.50}}$ & $\mathbf{58.78_{\pm 2.59}}$&{\quad}
  & $31.03_{\pm 0.81}$ & $25.65_{\pm 0.52}$ \\

\quad \textsc{HyCal}(Ours)
  & $\mathbf{38.32_{\pm 2.02}}$ & $\mathbf{50.27_{\pm 2.52}}$ & $\underline{77.42_{\pm 5.24}}$ & $\mathbf{69.59_{\pm 0.23}}$
  & $\mathbf{41.67_{\pm 0.34}}$ & $\underline{76.59_{\pm 3.27}}$ & $\underline{95.69_{\pm 0.38}}$ & $\underline{58.42_{\pm 2.72}}$&{\quad}
  & $\mathbf{54.96_{\pm 0.81}}$ & $19.87_{\pm 0.65}$ \\\midrule\hline

\multicolumn{11}{l}{\textbf{Last Acc.}}\\

\quad Primal-RAIL~\cite{xu2024advancing}
  & \textcolor{gray}{$19.90_{\pm 8.04}$}
  & \textcolor{gray}{$60.32_{\pm 0.59}$}
  & \textcolor{gray}{$76.76_{\pm 6.69}$}
  & \textcolor{gray}{$68.03_{\pm 1.51}$}
  & \textcolor{gray}{$37.23_{\pm 1.14}$}
  & \textcolor{gray}{$71.89_{\pm 1.64}$}
  & \textcolor{gray}{$90.46_{\pm 0.54}$}
  & \textcolor{gray}{$53.83_{\pm 3.88}$}&{\quad}
  & \textcolor{gray}{$59.80_{\pm 1.18}$}
  & \textcolor{gray}{$22.69_{\pm 2.04}$}
\\\hdashline

\quad FeCAM~\cite{goswami2023fecam}
  & $0.00_{\pm 0.00}$ & $0.00_{\pm 0.00}$ & $0.00_{\pm 0.00}$ & $31.77_{\pm 5.68}$
  & $11.88_{\pm 0.98}$ & $0.00_{\pm 0.00}$ & $1.76_{\pm 3.04}$ & $0.00_{\pm 0.00}$&{\quad}
  & $5.68_{\pm 0.73}$  & $11.34_{\pm 1.83}$ \\

\quad RanPAC~\cite{mcdonnell2023ranpac}
  & $\underline{37.42_{\pm 4.29}}$ & $\underline{40.41_{\pm 6.27}}$ & $76.91_{\pm 3.83}$ & $67.02_{\pm 1.02}$
  & $36.87_{\pm 1.04}$ & $71.24_{\pm 4.36}$ & $92.44_{\pm 0.49}$ & $56.07_{\pm 3.90}$&{\quad}
  & $59.80_{\pm 1.68}$ & $20.64_{\pm 0.71}$ \\

\quad KLDA~\cite{momeni2025continual}
  & $36.13_{\pm 2.28}$ & $37.82_{\pm 4.62}$ & $\mathbf{78.97_{\pm 5.62}}$ & $\mathbf{69.39_{\pm 1.70}}$
  & $\underline{39.86_{\pm 1.98}}$ & $\mathbf{76.74_{\pm 2.68}}$ & $\mathbf{95.77_{\pm 0.41}}$ & $\mathbf{58.78_{\pm 2.59}}$&{\quad}
  & $\underline{61.68_{\pm 1.24}}$ & $22.27_{\pm 0.89}$ \\

\quad \textsc{HyCal}(Ours)
  & $\mathbf{38.15_{\pm 1.82}}$ & $\mathbf{49.66_{\pm 2.42}}$ & $\underline{77.72_{\pm 5.20}}$ & $\underline{69.27_{\pm 0.43}}$
  & $\mathbf{41.67_{\pm 0.36}}$ & $\underline{76.55_{\pm 3.27}}$ & $\underline{95.69_{\pm 0.38}}$ & $\underline{58.42_{\pm 2.72}}$&{\quad}
  & $\mathbf{63.39_{\pm 0.48}}$ & $19.98_{\pm 0.64}$ \\


\midrule \midrule
\rowcolor{gray!24}
\multicolumn{12}{c}{\textit{Random Order II}} 
\\\midrule
& \rotatebox{0}{EuroSAT}
& \rotatebox{0}{Organ}
& \rotatebox{0}{ArtBench}
& \rotatebox{0}{MNIST}
& \rotatebox{0}{Flower}
& \rotatebox{0}{Aircraft}
& \rotatebox{0}{Galaxy}
& \rotatebox{0}{DTD}&{\quad}
& \rotatebox{0}{\textbf{Average}}
& \rotatebox{0}{$\boldsymbol{\sigma}$} \\ \toprule

\multicolumn{1}{l}{\quad Zero-shot}
  & 37.58 & 17.97 & 50.88 & 44.01 & 67.40 & 23.91 & 9.80 & 41.90&{\quad} & 56.41 & 18.76
\\\midrule\hline

\multicolumn{11}{l}{\textbf{Average Acc.}}\\

\quad Primal-RAIL~\cite{xu2024advancing}
  & \textcolor{gray}{$72.84_{\pm 3.93}$}
  & \textcolor{gray}{$54.25_{\pm 3.81}$}
  & \textcolor{gray}{$60.03_{\pm 0.66}$}
  & \textcolor{gray}{$77.83_{\pm 6.33}$}
  & \textcolor{gray}{$91.75_{\pm 0.58}$}
  & \textcolor{gray}{$37.86_{\pm 0.67}$}
  & \textcolor{gray}{$22.22_{\pm 10.58}$}
  & \textcolor{gray}{$67.77_{\pm 1.44}$}&{\quad}
  & \textcolor{gray}{$65.78_{\pm 0.87}$}
  & \textcolor{gray}{$22.48_{\pm 2.61}$}
\\\hdashline

\quad FeCAM~\cite{goswami2023fecam}
  & $3.24_{\pm 1.05}$ & $2.67_{\pm 1.31}$ & $4.71_{\pm 1.08}$ & $0.00_{\pm 0.00}$
  & $23.22_{\pm 7.27}$ & $12.44_{\pm 1.70}$ & $0.00_{\pm 0.00}$ & $31.24_{\pm 8.07}$&{\quad}
  & $7.90_{\pm 0.25}$ & $11.80_{\pm 2.76}$ \\

\quad RanPAC~\cite{mcdonnell2023ranpac}
  & $\underline{72.90_{\pm 3.17}}$ & $\underline{56.52_{\pm 3.40}}$ & \underline{$45.90_{\pm 3.80}$}
  & $\mathbf{80.32_{\pm 6.70}}$ 
  & $\underline{91.77_{\pm 1.33}}$
  & $36.54_{\pm 0.14}$ & $\mathbf{39.21_{\pm 4.30}}$
  & $67.32_{\pm 1.28}$ &{\quad}
  & $\underline{65.28_{\pm 1.53}}$ & $20.14_{\pm 2.11}$ \\

\quad KLDA~\cite{momeni2025continual}
  & $39.06_{\pm 1.55}$ & $33.26_{\pm 1.15}$ & $27.34_{\pm 1.10}$ 
  & $64.01_{\pm 2.90}$ 
  & $90.94_{\pm 0.28}$ 
  & $\underline{38.64_{\pm 1.43}}$
  & $33.27_{\pm 7.29}$
  & $\underline{68.63_{\pm 0.92}}$&{\quad} 
  & $32.89_{\pm 0.74}$ & $22.53_{\pm 1.25}$ \\

\quad \textsc{HyCal}(Ours)
  & $\mathbf{73.33_{\pm 3.32}}$ & $\mathbf{56.56_{\pm 3.97}}$ & $\mathbf{51.28_{\pm 2.75}}$
  & $\underline{76.76_{\pm 5.23}}$  
  & $\mathbf{95.76_{\pm 0.51}}$ 
  & $\mathbf{41.01_{\pm 0.76}}$
  & $\underline{38.67_{\pm 5.34}}$
  & $\mathbf{69.41_{\pm 1.01}}$ &{\quad}
  & $\mathbf{65.71_{\pm 1.17}}$ & $19.59_{\pm 1.13}$ 
\\\midrule\hline


\multicolumn{11}{l}{\textbf{Last Acc.}}\\

\quad Primal-RAIL~\cite{xu2024advancing}
  & \textcolor{gray}{$68.99_{\pm 5.05}$}
  & \textcolor{gray}{$51.98_{\pm 3.55}$}
  & \textcolor{gray}{$59.91_{\pm 0.75}$}
  & \textcolor{gray}{$77.36_{\pm 6.73}$}
  & \textcolor{gray}{$90.42_{\pm 0.72}$}
  & \textcolor{gray}{$37.50_{\pm 0.73}$}
  & \textcolor{gray}{$20.28_{\pm 9.73}$}
  & \textcolor{gray}{$67.77_{\pm 1.44}$}&{\quad}
  & \textcolor{gray}{$59.28_{\pm 1.57}$}
  & \textcolor{gray}{$22.50_{\pm 2.69}$}
\\\hdashline

\quad FeCAM~\cite{goswami2023fecam}
  & $0.00_{\pm 0.00}$ & $0.00_{\pm 0.00}$ & $0.00_{\pm 0.00}$ & $0.00_{\pm 0.00}$
  & $1.12_{\pm 3.21}$ & $11.90_{\pm 1.72}$ & $0.00_{\pm 0.00}$ & $31.24_{\pm 8.07}$&{\quad}
  & $5.53_{\pm 1.31}$ & $11.21_{\pm 2.58}$ \\

\quad RanPAC~\cite{mcdonnell2023ranpac}
  & $70.02_{\pm 4.13}$ & $54.34_{\pm 3.54}$ & \underline{$41.03_{\pm 3.95}$}
  & $77.57_{\pm 7.28}$ 
  & $91.90_{\pm 1.03}$
  & $37.08_{\pm 0.81}$ & $\mathbf{39.08_{\pm 4.09}}$
  & $67.32_{\pm 1.28}$&{\quad}
  & $59.79_{\pm 1.38}$ & $20.20_{\pm 1.99}$ \\

\quad KLDA~\cite{momeni2025continual}
  & $\mathbf{75.43_{\pm 3.99}}$ 
  & $\underline{57.80_{\pm 2.93}}$ 
  & $37.35_{\pm 2.72}$
  & $\mathbf{78.22_{\pm 4.61}}$ 
  & $\underline{95.33_{\pm 0.58}}$
  & $\underline{39.15_{\pm 1.15}}$
  & $33.44_{\pm 7.53}$ 
  & $\underline{68.63_{\pm 0.92}}$&{\quad}
  & $\underline{60.67_{\pm 1.13}}$ & $22.59_{\pm 1.50}$ \\

\quad \textsc{HyCal}(Ours)
  & $\underline{74.01_{\pm 3.21}}$ 
  & $\mathbf{58.46_{\pm 4.30}}$
  & $\mathbf{49.65_{\pm 2.44}}$ 
  & $\underline{77.69_{\pm 5.15}}$ 
  & $\mathbf{95.71_{\pm 0.53}}$ 
  & $\mathbf{40.95_{\pm 0.74}}$
  & $\underline{38.18_{\pm 5.02}}$
  & $\mathbf{69.41_{\pm 1.01}}$ &{\quad}
  & $\mathbf{63.01_{\pm 1.08}}$ & $19.89_{\pm 1.11}$ \\

\midrule \midrule

\rowcolor{gray!24}
\multicolumn{12}{c}{\textit{Random Order III}} 
\\\midrule
& \rotatebox{0}{EuroSAT}
& \rotatebox{0}{Galaxy}
& \rotatebox{0}{Organ}
& \rotatebox{0}{Flower}
& \rotatebox{0}{DTD}
& \rotatebox{0}{MNIST}
& \rotatebox{0}{Aircraft}
& \rotatebox{0}{ArtBench}&{\quad}
& \rotatebox{0}{\textbf{Average}}
& \rotatebox{0}{$\boldsymbol{\sigma}$} \\ \toprule

\multicolumn{1}{l}{\quad Zero-shot}
  & 37.58 & 9.80 & 17.97 & 67.40 & 41.90 & 44.01 & 23.91 & 50.88&{\quad} & 56.41 & 18.76
\\\midrule\hline


\multicolumn{11}{l}{\textbf{Average Acc.}}\\

\quad Primal-RAIL~\cite{xu2024advancing}
  & \textcolor{gray}{$72.40_{\pm 4.22}$}
  & \textcolor{gray}{$25.96_{\pm 2.91}$}
  & \textcolor{gray}{$54.38_{\pm 1.74}$}
  & \textcolor{gray}{$91.44_{\pm 0.46}$}
  & \textcolor{gray}{$68.09_{\pm 1.01}$}
  & \textcolor{gray}{$80.68_{\pm 3.75}$}
  & \textcolor{gray}{$37.02_{\pm 0.93}$}
  & \textcolor{gray}{$59.24_{\pm 1.67}$}&{\quad}
  & \textcolor{gray}{$61.48_{\pm 1.27}$}
  & \textcolor{gray}{$21.89_{\pm 0.96}$}
\\\hdashline

\quad FeCAM~\cite{goswami2023fecam}
  & $5.44_{\pm 1.07}$ & $4.53_{\pm 0.67}$ & $3.20_{\pm 2.38}$ & $6.24_{\pm 1.93}$
  & $35.83_{\pm 9.74}$ & $0.00_{\pm 0.00}$ & $11.60_{\pm 2.93}$ & $0.00_{\pm 0.00}$&{\quad}
  & $10.15_{\pm 1.25}$ & $11.76_{\pm 3.11}$ \\

\quad RanPAC~\cite{mcdonnell2023ranpac}
  & $\underline{72.72_{\pm 1.99}}$
  & $\mathbf{39.72_{\pm 3.65}}$
  & $\underline{56.02_{\pm 3.30}}$
  & $\underline{91.60_{\pm 0.42}}$
  & $65.82_{\pm 0.36}$
  & $78.17_{\pm 4.58}$
  & ${35.60_{\pm 0.78}}$
  & $\underline{39.72_{\pm 2.28}}$&{\quad}
  & $\underline{63.91_{\pm 1.21}}$
  & $20.60_{\pm 0.55}$ \\
\quad KLDA
  & $46.86_{\pm 1.94}$
  & $23.35_{\pm 1.93}$
  & $46.52_{\pm 3.05}$
  & $86.14_{\pm 2.78}$
  & $\underline{67.35_{\pm 1.40}}$
  & $\mathbf{81.35_{\pm 2.06}}$
  & $\underline{39.46_{\pm 0.66}}$
  & $37.17_{\pm 2.54}$
  &{\quad}
  & $39.60_{\pm 1.12}$
  & $22.38_{\pm 0.93}$ \\

\quad \textsc{HyCal}(Ours)
  & $\mathbf{73.81_{\pm 3.36}}$
  & $\underline{38.20_{\pm 5.59}}$
  & $\mathbf{58.84_{\pm 1.22}}$
  & $\mathbf{95.64_{\pm 0.65}}$
  & $\mathbf{69.54_{\pm 0.60}}$
  & $\underline{79.06_{\pm 3.41}}$
  & $\mathbf{41.43_{\pm 1.16}}$
  & $\mathbf{48.51_{\pm 1.85}}$&{\quad}
  & $\mathbf{64.73_{\pm 1.70}}$
  & $20.00_{\pm 1.69}$ 
\\\midrule\hline


\multicolumn{11}{l}{\textbf{Last Acc.}}\\

\quad Primal-RAIL~\cite{xu2024advancing}
  & \textcolor{gray}{$69.10_{\pm 5.00}$}
  & \textcolor{gray}{$19.83_{\pm 2.46}$}
  & \textcolor{gray}{$53.43_{\pm 1.22}$}
  & \textcolor{gray}{$90.40_{\pm 0.49}$}
  & \textcolor{gray}{$67.51_{\pm 0.90}$}
  & \textcolor{gray}{$80.49_{\pm 3.89}$}
  & \textcolor{gray}{$36.97_{\pm 0.94}$}
  & \textcolor{gray}{$59.24_{\pm 1.67}$}&{\quad}
  & \textcolor{gray}{$59.62_{\pm 0.55}$}
  & \textcolor{gray}{$22.94_{\pm 1.03}$}
\\\hdashline

\quad FeCAM~\cite{goswami2023fecam}
  & $0.00_{\pm 0.00}$ & $0.00_{\pm 0.00}$ & $0.00_{\pm 0.00}$ & $0.57_{\pm 0.51}$
  & $35.83_{\pm 9.74}$ & $0.00_{\pm 0.00}$ & $11.60_{\pm 2.93}$ & $0.00_{\pm 0.00}$&{\quad}
  & $6.00_{\pm 1.11}$ & $12.75_{\pm 3.18}$ \\

\quad RanPAC~\cite{mcdonnell2023ranpac}
  & ${70.36_{\pm 4.04}}$
  & $\mathbf{39.95_{\pm 4.09}}$
  & $55.10_{\pm 3.53}$
  & $91.92_{\pm 0.68}$
  & $66.92_{\pm 0.68}$
  & $77.37_{\pm 5.98}$
  & $36.84_{\pm 1.31}$
  & $\underline{39.72_{\pm 2.28}}$&{\quad}
  & $59.77_{\pm 1.51}$
  & $20.24_{\pm 1.12}$ \\
\quad KLDA
  & $\mathbf{75.57_{\pm 4.02}}$
  & $35.46_{\pm 3.58}$
  & $\underline{57.51_{\pm 3.50}}$
  & $\underline{95.10_{\pm 1.51}}$
  & $\underline{68.11_{\pm 1.35}}$
  & $\mathbf{80.69_{\pm 2.11}}$
  & $\underline{39.44_{\pm 0.81}}$
  & $37.17_{\pm 2.54}$
  &{\quad}
  & $\underline{61.13_{\pm 0.52}}$
  & $22.44_{\pm 0.80}$ \\

\quad \textsc{HyCal}(Ours)
  & $\underline{74.01_{\pm 3.22}}$
  & $\underline{38.09_{\pm 5.88}}$
  & $\mathbf{58.81_{\pm 1.27}}$
  & $\mathbf{95.62_{\pm 0.64}}$
  & $\mathbf{69.65_{\pm 0.65}}$
  & $\underline{79.07_{\pm 3.42}}$
  & $\mathbf{41.44_{\pm 1.16}}$
  & $\mathbf{48.51_{\pm 1.85}}$&{\quad}
  & $\mathbf{63.15_{\pm 0.74}}$
  & $20.04_{\pm 1.76}$ \\

\midrule\midrule
\rowcolor{gray!24}
\multicolumn{12}{c}{\textit{Random order IV}} 
\\\midrule
& \rotatebox{0}{ArtBench}
& \rotatebox{0}{MNIST}
& \rotatebox{0}{Organ}
& \rotatebox{0}{Aircraft}
& \rotatebox{0}{EuroSAT}
& \rotatebox{0}{DTD}
& \rotatebox{0}{Galaxy}
& \rotatebox{0}{Flower}&{\quad}
& \rotatebox{0}{\textbf{Average}}
& \rotatebox{0}{$\boldsymbol{\sigma}$} \\ \toprule

\multicolumn{1}{l}{\quad Zero-shot}
  & $50.88$ & $44.01$ & $17.97$ & $23.91$ & $37.58$ & $41.90$ & $9.80$ & $67.40$&{\quad} & $56.41$ & $18.76$
\\\midrule\hline


\multicolumn{11}{l}{\textbf{Average Acc.}}\\

\quad Primal-RAIL
  & \textcolor{gray}{$60.55_{\pm 0.19}$}
  & \textcolor{gray}{$77.25_{\pm 5.23}$}
  & \textcolor{gray}{$55.28_{\pm 5.19}$}
  & \textcolor{gray}{$37.01_{\pm 0.37}$}
  & \textcolor{gray}{$74.29_{\pm 7.47}$}
  & \textcolor{gray}{$68.54_{\pm 0.96}$}
  & \textcolor{gray}{$24.15_{\pm 8.06}$}
  & \textcolor{gray}{$90.49_{\pm 0.32}$}&{\quad}
  & \textcolor{gray}{$61.58_{\pm 0.48}$}
  & \textcolor{gray}{$21.96_{\pm 1.96}$}
\\\hdashline

\quad FeCAM
  & $0.00_{\pm 0.00}$ 
  & $0.00_{\pm 0.00}$
  & $0.00_{\pm 0.00}$
  & $11.93_{\pm 0.00}$
  & $0.00_{\pm 0.00}$
  & $33.17_{\pm 3.67}$
  & $0.00_{\pm 0.00}$
  & $0.14_{\pm 0.20}$&{\quad}
  & $5.66_{\pm 0.40}$
  & $11.89_{\pm 1.12}$ \\

\quad RanPAC
  & $\underline{39.93_{\pm 4.15}}$
  & $\mathbf{79.06_{\pm 7.55}}$
  & $\underline{56.45_{\pm 6.11}}$
  & $\mathbf{36.22_{\pm 2.13}}$
  & $\underline{71.24_{\pm 3.93}}$
  & $67.45_{\pm 1.89}$
  & $\mathbf{39.71_{\pm 0.89}}$
  & $92.43_{\pm 0.80}$&{\quad}
  & $\underline{56.08_{\pm 2.20}}$
  & $20.77_{\pm 1.02}$ \\
\quad KLDA~\cite{momeni2025continual}
  & $21.56_{\pm 1.54}$
  
  & $49.22_{\pm 5.31}$
  & $40.30_{\pm 2.21}$
  & $27.68_{\pm 0.45}$
  & $70.05_{\pm 4.84}$
  & $\underline{68.01_{\pm 1.04}}$
  & $36.23_{\pm 5.53}$
  & $\underline{95.26_{\pm 1.04}}$ &{\quad}
  & $31.44_{\pm 1.53}$
  & $25.01_{\pm 1.09}$
  
  \\

\quad \textsc{HyCal} (Ours)
  & $\mathbf{50.94_{\pm 2.66}}$
  & $76.84_{\pm 5.18}$
  & $\mathbf{58.46_{\pm 3.45}}$
  & $\mathbf{41.21_{\pm 0.81}}$
  & $\mathbf{72.29_{\pm 4.85}}$
  & $\mathbf{69.56_{\pm 1.24}}$
  & $\underline{39.35_{\pm 4.00}}$
  & $\mathbf{95.62_{\pm 0.41}}$&{\quad}
  & $\mathbf{59.69_{\pm 1.70}}$
  & $19.29_{\pm 1.31}$
\\\midrule\hline


\multicolumn{11}{l}{\textbf{Last Acc.}}\\

\quad Primal-RAIL
  & \textcolor{gray}{$60.55_{\pm 0.08}$}
  & \textcolor{gray}{$76.69_{\pm 5.99}$}
  & \textcolor{gray}{$54.19_{\pm 4.98}$}
  & \textcolor{gray}{$36.60_{\pm 0.56}$}
  & \textcolor{gray}{$73.05_{\pm 7.54}$}
  & \textcolor{gray}{$67.75_{\pm 0.94}$}
  & \textcolor{gray}{$22.90_{\pm 7.91}$}
  & \textcolor{gray}{$90.49_{\pm 0.32}$}&{\quad}
  & \textcolor{gray}{$60.28_{\pm 0.61}$}
  & \textcolor{gray}{$22.17_{\pm 1.98}$}
\\\hdashline

\quad FeCAM
  & $8.56_{\pm 4.51}$
  & $0.00_{\pm 0.00}$
  & $0.95_{\pm 1.74}$
  & $12.31_{\pm 2.19}$
  & $0.00_{\pm 0.00}$
  & $33.19_{\pm 3.67}$
  & $0.00_{\pm 0.00}$
  & $0.14_{\pm 0.20}$&{\quad}
  & $8.64_{\pm 3.29}$
  & $11.67_{\pm 1.08}$ \\

\quad RanPAC
  & $\underline{38.00_{\pm 2.78}}$
  & $\underline{78.88_{\pm 7.08}}$
  & $57.08_{\pm 3.78}$
  & $37.43_{\pm 1.28}$
  & $72.07_{\pm 5.04}$
  & $67.26_{\pm 1.38}$
  & $\mathbf{39.35_{\pm 1.31}}$
  & ${92.43_{\pm 0.80}}$&{\quad}
  & $60.31_{\pm 0.63}$
  & $20.91_{\pm 0.94}$ \\

\quad KLDA~\cite{momeni2025continual}
  & $37.59_{\pm 4.24}$
  & $\mathbf{79.63_{\pm 4.88}}$
  & $\underline{58.08_{\pm 2.55}}$
  & $\underline{39.58_{\pm 0.75}}$
  & $\mathbf{72.63_{\pm 6.17}}$
  & $\underline{68.79_{\pm 0.94}}$
  & $36.33_{\pm 5.01}$
  & $\underline{95.26_{\pm 1.04}}$ &{\quad}
  & $\underline{60.99_{\pm 0.68}}$
  & $21.94_{\pm 1.24}$ \\
\quad \textsc{HyCal} (Ours)
  & $\mathbf{49.42_{\pm 2.29}}$
  & $77.68_{\pm 5.13}$
  & $\mathbf{59.06_{\pm 2.58}}$
  & $\mathbf{41.25_{\pm 0.85}}$
  & $\underline{72.40_{\pm 4.78}}$
  & $\mathbf{69.47_{\pm 1.25}}$
  & $\underline{39.30_{\pm 3.83}}$
  & $\mathbf{95.62_{\pm 0.41}}$&{\quad}
  & $\mathbf{63.03_{\pm 1.11}}$
  & $19.50_{\pm 1.27}$ \\
\bottomrule

\end{tabular}
}
\end{table*}

\subsection{Hyperparameter sensitivity}

\subsubsection{Sensitivity to fusion weights $\alpha$ and $\beta$}
We examine the influence of the hyperparameters $\alpha$ and $\beta$ by evaluating performance across the search ranges $\alpha \in \{1,10,20,40,60,80\}$ and $\beta \in \{0,5,10\}$. Overall, \textsc{HyCal} shows stable performance across this hyperparameter space, with calibration strength varying smoothly as the activation point and smoothness parameters change. The effect of varying $\alpha$ and $\beta$ is visualized in Fig.~\ref{fig:hyperparmeter}.

Across the examined range, performance is generally robust with respect to $\alpha$, indicating that the calibration activation point does not strongly influence the overall behavior of the method. In contrast, $\beta$ exhibits a clearer trend: larger values relax the threshold that determines when Mahalanobis distance becomes active, effectively reducing its contribution and allowing cosine similarity to dominate. As a result, performance consistently drops when moving from small values such as $\beta=0,5$ to higher values like $\beta=10$, indicating that weakening this threshold diminishes the benefit obtained from the Mahalanobis component.

\subsubsection{Sensitivity to covariance regularization ($\lambda$ and $\gamma$)}

We also examine the influence of the covariance-regularization parameters $\lambda$ and $\gamma$ in \cref{eq:reg_covariance}. 
For this analysis, we fix $\alpha = 10$ and $\beta = 5$, and evaluate the method under the High-scale domain imbalance setting. We use $\lambda = 10^{-4}$ and $\gamma = 1$ by default, and vary $\lambda$ from $10^{-1}$ to $10^{-5}$ and $\gamma$ from $1$ to $2 \times 10^{-3}$, as shown in \cref{fig:hyperparmeter_lambda_gamma}. 
Overall, performance remains stable across a wide range of moderate settings, indicating that \textsc{HyCal} does not require delicate tuning of these parameters.

\subsection{Fusion strategies}

We compare two image–text fusion strategies in \cref{tab:fusion}: summation, which preserves the 512-dimensional CLIP embedding, and concatenation, which corresponds to the 1024-dimensional representation. Across all settings, the sum-based fusion consistently achieves better performance than the concatenation fusion.

\section{Efficiency analysis}
Because \textsc{HyCal} keeps the pretrained backbone frozen and updates only class prototypes and regularized precision matrices for newly introduced classes, its computational cost scales with the number of new classes rather than with full model retraining. Unlike prior approaches that recompute statistics over all classes or domains in XD-VSCIL, \textsc{HyCal} computes and stores covariance-related statistics only for newly added classes while reusing previously constructed prototypes. This design yields linear growth in both memory and computation as the class set expands, making the method well suited to incremental deployment.

This efficiency is also reflected in practice. When 10 new classes are added, \textsc{HyCal} requires only 5.4 GFLOPs, whereas full retraining in Primal-RAIL requires 12 GFLOPs. The required storage remains moderate, amounting to 301 MB when maintaining statistics for $C = 300$ classes. In addition, prototype construction and covariance estimation are independent across classes, enabling straightforward parallelization across newly introduced categories.  Taken together, these properties make \textsc{HyCal} an efficient and scalable alternative for XD-VSCIL, especially in settings where new classes or domains must be incorporated quickly without repeated parameter optimization.

\section{Detailed experimental setting}
\label{sec:experimental_setting}

\subsection{Implementation details}

We use the Vision Transformer (ViT-B/16) model with a frozen CLIP text encoder for all experiments. The model weights are loaded from the \texttt{openai/clip-vit-base-patch16} checkpoint. The image encoder's parameters are kept frozen throughout all experiments. All experiments were conducted using PyTorch version 2.1.1. All training and evaluation runs were performed on a single NVIDIA GeForce RTX 3090 GPU. All images are resized to $224 \times 224$ and normalized using standard CLIP statistics.

\textsc{HyCal} is implemented on frozen \texttt{CLIP ViT-B/16}~\cite{CLIP} features using the original 512-dimensional embedding, whereas Primal-RAIL and other training-free baselines operate on 1024-dimensional embeddings. 
This setup ensures a fair comparison by contrasting \textsc{HyCal}'s training-free, low-dimensional configuration against both training-free and fully trainable alternatives.

To ensure the robustness of our results, all experiments were run four times with different random seeds: \{0, 1, 42, 1993\}. These seeds control the initialization of any trainable parameters, the selection of few-shot examples, and the order of data shuffling during training. The results reported in the main paper are the mean and standard deviation over these four runs.

\textsc{HyCal} uses two fusion weights hyperparameters, $\alpha$ and $\beta$. 
We perform a lightweight grid search over $\alpha \in \{1,10,20,40,60,80\}$ and $\beta \in \{0,5,10\}$ and select the best configuration for each setting. 
For XD-VSCIL, we set $(\alpha,\beta)=(20,5)$ for the Balanced-in-class domain setting, $(60,5)$ for the Cross-scale imbalance setting, and $(10,5)$ for the High-scale domain imbalance setting. For FSCIL, we use $(1,0)$ in the 5-shot case to allow immediate activation under extremely sparse data, and $(10,5)$ for the 10-, 15-, and 20-shot settings. We use $\lambda=10^{-4}$ and $\gamma = 1$ in all reported experiments.

\subsection{Evaluation metrics}

We report three metrics: Last accuracy, Average accuracy, and the proposed $\text{S}_{\text{CDE}}$. As $\text{S}_{\text{CDE}}$ is formally defined and analyzed in the main text, we provide concise descriptions of the remaining two metrics. Last accuracy denotes the Top-1 accuracy on the final task after all incremental stages and reflects the model's ability to retain previously learned knowledge. Average accuracy is computed as the mean Top-1 accuracy across all tasks,
\[
\text{Avg Acc.} = \frac{1}{\mathcal{T}} \sum_{t=1}^{\mathcal{T}} \text{Acc}(t),
\]
where $\mathcal{T}$ is the total number of steps, and provides a summary of overall performance throughout the entire learning trajectory.

\subsection{XD-VSCIL settings}
We evaluate our training-free hybrid calibration method \textsc{HyCal} by comparing it against both training-free and trainable baselines. As training-free approaches, we include FeCAM~\cite{goswami2023fecam}, RanPAC~\cite{mcdonnell2023ranpac}, and KLDA~\cite{momeni2025continual}, which represent recent state-of-the-art methods relying solely on frozen visual encoders. 
For the trainable baseline, we adopt Primal-RAIL~\cite{xu2024advancing} and report its performance as a reference point for methods that fine-tune task-specific modules. 

The training-sample distributions for three XD-VSCIL settings are illustrated in Fig.~\ref{fig:exp_setting}(a)--(c), which highlight the distinct balance/imbalance patterns observed in each setting. In the Cross-scale imbalance setting, a variable number of training samples per class ($K_c$) is used, and the exact values depend on the random seed. For reference, \cref{tab:xd_vscil_k_1} and \cref{tab:xd_vscil_k_2} report the specific $K_c$ values employed when the random seed is set to 42.

\subsubsection{FSCIL setting.} 
We follow the standard few-shot evaluation protocol. For each dataset, we conduct experiments in 5, 10, 15, 20-shot settings. In a $K$-shot setting, we randomly sample $K$ images per class from the original training set.

\begin{table*}[!ht]
\caption{Per-class sample counts ($K_c$) for the XD-VSCIL setting on the Aircraft, ArtBench, DTD, and EuroSAT datasets, with \texttt{random\_seed=42}. Classes are sorted alphabetically.}
\label{tab:xd_vscil_k_1}
\centering
\setlength{\tabcolsep}{3mm}
\resizebox{0.98\textwidth}{!}{%
\begin{tabular}{l c l c l c l c}
\toprule\midrule
\multicolumn{8}{c}{\textbf{Aircraft (aeronautics)}~\cite{Aircraft}} \\
\midrule
Class Name & $K_c$ & Class Name & $K_c$ & Class Name & $K_c$ & Class Name & $K_c$ \\
\midrule
707-320 & 45 & 727-200 & 12 & 737-200 & 6 & 737-300 & 22 \\
737-400 & 20 & 737-500 & 19 & 737-600 & 13 & 737-700 & 11 \\
737-800 & 48 & 737-900 & 39 & 747-100 & 10 & 747-200 & 42 \\
747-300 & 32 & 747-400 & 7 & 757-200 & 6 & 757-300 & 10 \\
767-200 & 18 & 767-300 & 19 & 767-400 & 37 & 777-200 & 43 \\
777-300 & 6 & A300B4 & 40 & A310 & 17 & A318 & 50 \\
A319 & 46 & A320 & 49 & A321 & 39 & A330-200 & 31 \\
A330-300 & 19 & A340-200 & 33 & A340-300 & 42 & A340-500 & 22 \\
A340-600 & 5 & A380 & 15 & ATR-42 & 49 & ATR-72 & 32 \\
An-12 & 26 & BAE 146-200 & 22 & BAE 146-300 & 14 & BAE-125 & 18 \\
Beechcraft 1900 & 26 & Boeing 717 & 11 & C-130 & 10 & C-47 & 29 \\
CRJ-200 & 11 & CRJ-700 & 27 & CRJ-900 & 27 & Cessna 172 & 43 \\
Cessna 208 & 21 & Cessna 525 & 7 & Cessna 560 & 34 & Challenger 600 & 39 \\
DC-10 & 12 & DC-3 & 29 & DC-6 & 10 & DC-8 & 40 \\
DC-9-30 & 23 & DH-82 & 45 & DHC-1 & 44 & DHC-6 & 28 \\
DHC-8-100 & 41 & DHC-8-300 & 17 & DR-400 & 50 & Dornier 328 & 9 \\
E-170 & 7 & E-190 & 47 & E-195 & 19 & EMB-120 & 23 \\
ERJ 135 & 10 & ERJ 145 & 19 & Embraer Legacy 600 & 11 & Eurofighter Typhoon & 29 \\
F-16A/B & 22 & F/A-18 & 34 & Falcon 2000 & 45 & Falcon 900 & 28 \\
Fokker 100 & 15 & Fokker 50 & 28 & Fokker 70 & 27 & Global Express & 18 \\
Gulfstream IV & 47 & Gulfstream V & 22 & Hawk T1 & 49 & Il-76 & 48 \\
L-1011 & 46 & MD-11 & 9 & MD-80 & 43 & MD-87 & 45 \\
MD-90 & 15 & Metroliner & 39 & Model B200 & 20 & PA-28 & 15 \\
SR-20 & 34 & Saab 2000 & 29 & Saab 340 & 22 & Spitfire & 45 \\
Tornado & 49 & Tu-134 & 40 & Tu-154 & 19 & Yak-42 & 48 \\
\midrule\midrule
\multicolumn{8}{c}{\textbf{ArtBench (art)}~\cite{ArtBench-10}  } \\
\midrule
Class Name & $K_c$ & Class Name & $K_c$ & Class Name & $K_c$ & Class Name & $K_c$ \\
\midrule
Art Nouveau   & 46 & Impressionism      & 13 & Renaissance & 50 & Ukiyo-e    & 48 \\
Baroque       & 13 & Post-impressionism & 20 & Romanticism & 25 &            &    \\
Expressionism & 49 & Realism            & 14 & Surrealism  & 20 &            &    \\
\midrule\midrule
\multicolumn{8}{c}{\textbf{DTD (textures)}~\cite{DTD}  } \\
\midrule
Class Name & $K_c$ & Class Name & $K_c$ & Class Name & $K_c$ & Class Name & $K_c$ \\
\midrule
banded       & 30 & flecked     & 10 & matted       & 34 & sprinkled  & 22 \\
blotchy      & 36 & freckled    & 30 & meshed       & 45 & stained    & 25 \\
braided      & 14 & frilly      & 37 & paisley      & 47 & stratified & 23 \\
bubbly       & 41 & gauzy       & 22 & perforated   & 49 & striped    & 26 \\
bumpy        & 45 & grid        & 49 & pitted       & 27 & studded    & 40 \\
chequered    & 22 & grooved     & 28 & pleated      & 9  & swirly     & 41 \\
cobwebbed    & 45 & honeycombed & 34 & polka-dotted & 39 & veined     & 13 \\
cracked      & 31 & interlaced  & 36 & porous       & 30 & waffled    & 41 \\
crosshatched & 29 & knitted     & 25 & potholed     & 18 & woven      & 36 \\
crystalline  & 33 & lacelike    & 42 & scaly        & 32 & wrinkled   & 26 \\
dotted       & 9  & lined       & 5  & smeared      & 18 & zigzagged  & 48 \\
fibrous      & 45 & marbled     & 10 & spiralled    & 36 &            &   \\ 

\midrule\midrule
\multicolumn{8}{c}{\textbf{EuroSAT (remote sensing)}~\cite{EuroSAT} } \\
\midrule
Class Name & $K_c$ & Class Name & $K_c$ & Class Name & $K_c$ & Class Name & $K_c$ \\
\midrule
Annual Crop Land           & 26 & Highway or Road      & 19 & Permanent Crop Land   & 40 & Sea or Lake & 35 \\
Forest                        & 28 & Industrial Buildings  & 34 & Residential Buildings  & 44 &                                &    \\
Herbaceous Vegetation Land  & 8  & Pasture Land          & 26 & River                  & 43 &                                &    \\
\bottomrule
\end{tabular}
}
\end{table*}

\begin{table*}[!ht]
\centering

\caption{Per-class sample counts ($K_c$) for the XD-VSCIL setting on the Galaxy, MNIST, OrganMNIST, and OxfordFlowers, with \texttt{random\_seed=42}. Classes are sorted alphabetically.}
\label{tab:xd_vscil_k_2}
\setlength{\tabcolsep}{3mm}
\resizebox{0.98\textwidth}{!}{%
\begin{tabular}{l c l c l c l c}
\toprule\midrule
\multicolumn{8}{c}{\textbf{Galaxy (astronomy)}~\cite{Galaxy10}  } \\
\midrule
Class Name & $K_c$ & Class Name & $K_c$ & Class Name & $K_c$ & Class Name & $K_c$ \\
\midrule
Disturbed Galaxies    & 7  & \begin{tabular}[c]{@{}l@{}}In-between Round\\ Smooth Galaxies\end{tabular}  & 27 & \begin{tabular}[c]{@{}l@{}}Unbarred Tight\\ Spiral Galaxies\end{tabular}  & 23 & \begin{tabular}[c]{@{}l@{}}Edge-on Galaxies\\ with Bulge\end{tabular}  & 43 \\
Merging Galaxies      & 49 & \begin{tabular}[c]{@{}l@{}}Cigar Shaped\\ Smooth Galaxies\end{tabular}     & 37 & \begin{tabular}[c]{@{}l@{}}Unbarred Loose\\ Spiral Galaxies\end{tabular}  & 48 &                             &    \\
\begin{tabular}[c]{@{}l@{}}Round Smooth\\ Galaxies\end{tabular}  & 9  & \begin{tabular}[c]{@{}l@{}} Barred Spiral\\ Galaxies\end{tabular}  & 11 & \begin{tabular}[c]{@{}l@{}} Edge-on Galaxies\\ without Bulge \end{tabular}   & 24 &                             &    \\

\midrule\midrule
\multicolumn{8}{c}{\textbf{MNIST(fundamental visual perception)}~\cite{MNIST} } \\
\midrule
Class Name & $K_c$ & Class Name & $K_c$ & Class Name & $K_c$ & Class Name & $K_c$ \\
\midrule
0                     & 9  & 3                                & 29 & 6                              & 6  & 9                           & 37 \\
1                     & 19 & 4                                & 13 & 7                              & 17 &                             &    \\
2                     & 29 & 5                                & 11 & 8                              & 38 &                             &    \\

\midrule\midrule
\multicolumn{8}{c}{\textbf{OrganMNIST (medical imaging)}~\cite{OrganMNIST2} } \\
\midrule
Class Name & $K_c$ & Class Name & $K_c$ & Class Name & $K_c$ & Class Name & $K_c$ \\
\midrule
liver                 & 48 & pancreas                         & 43 & bladder                        & 35 & right femur                 & 7  \\
right lung            & 31 & left lung                        & 11 & spleen                         & 44 & left femur                  & 39 \\
right kidney          & 14 & left kidney                      & 16 & heart                          & 43 &                             &    \\

\midrule\midrule
\multicolumn{8}{c}{\textbf{OxfordFlowers (fine-grained biology)}~\cite{Flowers102} } \\
\midrule
Class Name & $K_c$ & Class Name & $K_c$ & Class Name & $K_c$ & Class Name & $K_c$ \\
\midrule
passion flower        & 42 & fritillary                       & 16 & sunflower                      & 24 & pink primrose               & 12 \\
water lily            & 11 & sweet william                    & 41 & magnolia                       & 21 & fire lily                   & 27 \\
cyclamen              & 25 & azalea                           & 13 & osteospermum                   & 40 & red ginger                  & 48 \\
watercress            & 28 & primula                          & 41 & garden phlox                   & 41 & prince of wales feathers    & 43 \\
frangipani            & 44 & cape flower                      & 48 & sweet pea                      & 8  & carnation                   & 17 \\
wallflower            & 21 & purple coneflower                & 19 & daffodil                       & 30 & mexican aster               & 41 \\
rose                  & 16 & colt's foot                      & 11 & king protea                    & 11 & alpine sea holly            & 41 \\
petunia               & 40 & artichoke                        & 22 & great masterwort               & 49 & siam tulip                  & 46 \\
poinsettia            & 21 & wild pansy                       & 20 & black-eyed susan               & 15 & spring crocus               & 46 \\
clematis              & 20 & peruvian lily                    & 40 & bearded iris                   & 8  & globe thistle               & 39 \\
hibiscus              & 22 & ruby-lipped cattleya             & 32 & windflower                     & 18 & bolero deep blue            & 16 \\
lotus                 & 35 & canna lily                       & 29 & ball moss                      & 21 & tiger lily                  & 23 \\
anthurium             & 44 & gazania                          & 17 & spear thistle                  & 35 & moon orchid                 & 27 \\
thorn apple           & 26 & lenten rose                      & 14 & silverbush                     & 45 & gaura                       & 25 \\
barbeton daisy        & 28 & buttercup                        & 14 & balloon flower                 & 11 & japanese anemone            & 48 \\
sword lily            & 16 & pelargonium                      & 35 & oxeye daisy                    & 33 & foxglove                    & 37 \\
morning glory         & 36 & desert-rose                      & 5  & cautleya spicata               & 26 & bougainvillea               & 36 \\
columbine             & 32 & hippeastrum                      & 16 & common dandelion               & 21 & camellia                    & 33 \\
geranium              & 18 & giant white arum lily            & 31 & yellow iris                    & 28 & mallow                      & 10 \\
bishop of llandaff    & 31 & marigold                         & 36 & monkshood                      & 28 & mexican petunia             & 29 \\
tree mallow           & 49 & orange dahlia                    & 13 & love in the mist               & 40 & bromelia                    & 44 \\
pink-yellow dahlia    & 28 & hard-leaved pocket orchid        & 45 & corn poppy                     & 40 & blanket flower              & 23 \\
bee balm              & 28 & english marigold                 & 38 & grape hyacinth                 & 21 & trumpet creeper             & 34 \\
snapdragon            & 30 & stemless gentian                 & 39 & canterbury bells               & 27 & blackberry lily             & 45 \\
californian poppy     & 22 & tree poppy                       & 25 & globe-flower                   & 6  &                             &    \\
bird of paradise      & 33 & pincushion flower                & 36 & toad lily                      & 31 &                             &   \\
\bottomrule
\end{tabular}
}
\end{table*}

\section{Numerical results of Balanced-in-class domain and Cross-scale imbalance settings}
\label{sec:result_bal_cross}
For completeness, we provide the numerical values for the Balanced-in-class domain and Cross-scale imbalance settings. The Balanced-in-Class Domain setting results are provided in \cref{tab:V1}, and the Cross-scale imbalance setting results are summarized in \cref{tab:high_scale_suppl}.

\begin{table*}[t]
\caption{Performance under Balanced-in-class domain setting.
Average and Last accuracy (\%) across 8 domains.
Results are averaged over 4 seeds with 95\% confidence intervals in subscripts. Domain names are abbreviated.}
\label{tab:V1}
\centering
\setlength{\tabcolsep}{0.9mm}
\resizebox{0.98\textwidth}{!}{%
\begin{tabular}{l c c c c c c c c c c c}
\toprule\midrule
\rowcolor{gray!8}
\multicolumn{12}{c}{\textbf{Balanced-in-Class Domain}} \\
\midrule
& Aircraft & ArtBench & DTD & EuroSAT & Galaxy & MNIST & Organ & Flower &{\quad}&  \textbf{Average} & $\boldsymbol{\sigma}$ \\
\toprule

\multicolumn{1}{l}{\quad Zero-shot}
  & $23.91$ & $50.88$ & $41.90$ & $37.58$ & $9.80$ & $44.01$ & $17.97$ & $67.40$  &{\quad}& $56.41$ & $18.76$
\\\midrule\hline

\multicolumn{12}{l}{\textbf{Average Acc.}}\\

\quad Primal-RAIL~\cite{xu2024advancing}
  & \textcolor{gray}{$32.64_{\pm 0.88}$}
  & \textcolor{gray}{$60.60_{\pm 0.67}$}
  & \textcolor{gray}{$64.54_{\pm 0.68}$}
  & \textcolor{gray}{$88.16_{\pm 0.73}$}
  & \textcolor{gray}{$51.55_{\pm 1.93}$}
  & \textcolor{gray}{$91.58_{\pm 0.70}$}
  & \textcolor{gray}{$67.94_{\pm 2.90}$}
  & \textcolor{gray}{$86.85_{\pm 0.56}$} &{\quad}
  & \textcolor{gray}{$56.57_{\pm 0.55}$}
  & \textcolor{gray}{$20.38_{\pm 0.46}$}
\\\hdashline

\quad FeCAM~\cite{goswami2023fecam}
  & $3.94_{\pm 0.67}$ & $40.69_{\pm 2.23}$ & $3.60_{\pm 2.67}$ & $17.51_{\pm 9.16}$
  & $17.46_{\pm 7.04}$ & $0.47_{\pm 1.14}$ & $9.72_{\pm 10.18}$ & $0.00_{\pm 0.00}$  &{\quad}
  & $14.88_{\pm 1.39}$ & $13.96_{\pm 0.63}$ \\

\quad RanPAC~\cite{mcdonnell2023ranpac}
  & $\underline{30.18_{\pm 1.99}}$
  & $\underline{58.88_{\pm 1.22}}$
  & $59.20_{\pm 2.93}$
  & $87.26_{\pm 1.57}$
  & $\mathbf{57.60_{\pm 1.11}}$
  & $93.85_{\pm 0.82}$
  & $72.28_{\pm 1.90}$
  & $87.26_{\pm 1.14}$ &{\quad}
  & $\underline{55.60_{\pm 1.10}}$
  & $21.15_{\pm 0.81}$ \\

\quad KLDA~\cite{momeni2025continual}
  & $24.63_{\pm 0.50}$
  & $50.53_{\pm 0.54}$
  & $\underline{60.76_{\pm 1.23}}$
  & $\mathbf{90.44_{\pm 0.53}}$
  & $\underline{55.85_{\pm 2.74}}$
  & $\underline{95.04_{\pm 0.66}}$
  & $\underline{74.94_{\pm 1.07}}$
  & $\underline{92.85_{\pm 0.53}}$ &{\quad}
  & $46.57_{\pm 0.62}$
  & $24.75_{\pm 0.38}$ \\

\quad \textsc{HyCal}(Ours)
  & $\mathbf{34.74_{\pm 1.69}}$
  & $\mathbf{60.41_{\pm 0.59}}$
  & $\mathbf{64.54_{\pm 1.03}}$
  & $\underline{89.25_{\pm 0.24}}$
  & $51.65_{\pm 1.80}$
  & $\mathbf{95.14_{\pm 0.49}}$
  & $\mathbf{76.36_{\pm 1.14}}$
  & $\mathbf{93.07_{\pm 0.75}}$ &{\quad}
  & $\mathbf{57.75_{\pm 0.45}}$
  & $21.63_{\pm 0.57}$
\\\midrule\hline

\multicolumn{11}{l}{\textbf{Last Acc.}}\\

\quad Primal-RAIL~\cite{xu2024advancing}
  & \textcolor{gray}{$31.77_{\pm 0.54}$}
  & \textcolor{gray}{$60.67_{\pm 0.56}$}
  & \textcolor{gray}{$62.56_{\pm 0.97}$}
  & \textcolor{gray}{$87.89_{\pm 0.80}$}
  & \textcolor{gray}{$50.70_{\pm 2.29}$}
  & \textcolor{gray}{$91.32_{\pm 0.83}$}
  & \textcolor{gray}{$67.73_{\pm 2.91}$}
  & \textcolor{gray}{$86.85_{\pm 0.56}$} &{\quad}
  & \textcolor{gray}{$67.44_{\pm 0.68}$}
  & \textcolor{gray}{$20.67_{\pm 0.44}$}
\\\hdashline

\quad FeCAM~\cite{goswami2023fecam}
  & $3.05_{\pm 0.93}$ & $40.69_{\pm 2.23}$ & $3.60_{\pm 2.67}$ & $17.51_{\pm 9.16}$
  & $17.46_{\pm 7.04}$ & $0.47_{\pm 1.14}$ & $9.72_{\pm 10.18}$ & $0.00_{\pm 0.00}$ &{\quad}
  & $11.56_{\pm 1.94}$ & $14.04_{\pm 0.62}$ \\

\quad RanPAC~\cite{mcdonnell2023ranpac}
  & $29.03_{\pm 1.88}$
  & $52.14_{\pm 24.34}$
  & $58.27_{\pm 2.60}$
  & $86.05_{\pm 2.59}$
  & $\mathbf{56.75_{\pm 1.44}}$
  & $92.83_{\pm 1.67}$
  & $70.92_{\pm 3.06}$
  & $87.26_{\pm 1.14}$ &{\quad}
  & $67.64_{\pm 0.62}$
  & $22.16_{\pm 3.21}$ \\

\quad KLDA~\cite{momeni2025continual}
  & $\underline{33.40_{\pm 1.00}}$
  & $\mathbf{61.49_{\pm 0.62}}$
  & $\underline{62.34_{\pm 1.56}}$
  & $\mathbf{90.13_{\pm 0.36}}$
  & $\underline{55.69_{\pm 2.79}}$
  & $\underline{94.78_{\pm 0.64}}$
  & $\underline{74.66_{\pm 1.17}}$
  & $\underline{92.85_{\pm 0.53}}$ &{\quad}
  & $\mathbf{70.67_{\pm 0.90}}$
  & $21.50_{\pm 0.57}$ \\

\quad \textsc{HyCal}(Ours)
  & $\mathbf{34.77_{\pm 1.72}}$
  & $\underline{60.70_{\pm 0.41}}$
  & $\mathbf{63.67_{\pm 1.21}}$
  & $\underline{89.22_{\pm 0.24}}$
  & $51.34_{\pm 1.59}$
  & $\mathbf{95.16_{\pm 0.47}}$
  & $\mathbf{76.41_{\pm 1.11}}$
  & $\mathbf{93.07_{\pm 0.75}}$ &{\quad}
  & $\underline{70.54_{\pm 0.23}}$
  & $21.68_{\pm 0.55}$
\\
\bottomrule

\end{tabular}
}
\end{table*}

\begin{table*}[t]
\caption{Performance under Cross-scale imbalance setting.
Average and Last accuracy (\%) across 8 domains. Results are averaged over 4 seeds with 95\% confidence intervals in subscripts. Domain names are abbreviated.}
\label{tab:high_scale_suppl}
\centering
\setlength{\tabcolsep}{0.9mm}
\resizebox{0.98\textwidth}{!}{%
\begin{tabular}{l c c c c c c c c c c c}
\toprule\midrule
\rowcolor{gray!16}
\multicolumn{12}{c}{\textbf{Cross-Scale Imbalance}} \\
\midrule
& Aircraft & ArtBench & DTD & EuroSAT & Galaxy & MNIST & Organ & Flower &{\quad} & \textbf{Average} & $\boldsymbol{\sigma}$ \\
\toprule

\multicolumn{1}{l}{\quad Zero-shot}
  & $23.91$ & $50.88$ & $41.90$ & $37.58$ & $9.80$ & $44.01$ & $17.97$ & $67.40$ &{\quad} & $56.41$ & $18.76$
\\\midrule\hline


\multicolumn{12}{l}{\textbf{Average Acc.}}\\

\quad Primal\mbox{-}RAIL
  & \textcolor{gray}{$36.72_{\pm 1.63}$}
  & \textcolor{gray}{$59.77_{\pm 1.10}$}
  & \textcolor{gray}{$65.60_{\pm 2.81}$}
  & \textcolor{gray}{$76.42_{\pm 3.42}$}
  & \textcolor{gray}{$36.49_{\pm 8.47}$}
  & \textcolor{gray}{$57.61_{\pm 2.17}$}
  & \textcolor{gray}{$53.59_{\pm 3.54}$}
  & \textcolor{gray}{$87.71_{\pm 3.18}$} &{\quad}
  & \textcolor{gray}{$53.12_{\pm 0.71}$}
  & \textcolor{gray}{$17.85_{\pm 2.14}$}
\\\hdashline

\quad FeCAM
  & $7.84_{\pm 5.90}$ & $29.20_{\pm 9.53}$ & $20.46_{\pm 2.70}$ & $35.55_{\pm 9.63}$
  & $15.42_{\pm 14.00}$ & $12.99_{\pm 17.42}$ & $18.38_{\pm 25.53}$ & $8.13_{\pm 4.92}$ &{\quad}
  & $19.14_{\pm 4.82}$ & $11.92_{\pm 2.54}$ \\

\quad RanPAC
  & $37.75_{\pm 3.73}$
  & $51.65_{\pm 2.78}$
  & $63.09_{\pm 5.28}$
  & $76.03_{\pm 6.98}$
  & $45.11_{\pm 5.95}$
  & $78.20_{\pm 12.34}$
  & $64.02_{\pm 7.98}$
  & $87.30_{\pm 7.11}$ &{\quad}
  & $53.43_{\pm 3.94}$
  & $17.40_{\pm 2.76}$ \\

\quad KLDA
  & $\underline{43.11_{\pm 0.85}}$
  & $\mathbf{55.68_{\pm 1.12}}$
  & $\underline{66.66_{\pm 3.95}}$
  & $\mathbf{85.61_{\pm 1.82}}$
  & $\underline{48.36_{\pm 5.27}}$
  & $\underline{89.17_{\pm 2.10}}$
  & $\underline{70.04_{\pm 1.51}}$
  & $\mathbf{96.11_{\pm 0.65}}$ &{\quad}
  & $\underline{58.60_{\pm 0.59}}$
  & $19.68_{\pm 0.69}$ \\

\quad \textsc{HyCal}(Ours)
  & $\mathbf{44.72_{\pm 1.91}}$
  & $\underline{54.09_{\pm 1.67}}$
  & $\mathbf{66.73_{\pm 2.16}}$
  & $\underline{82.30_{\pm 2.05}}$
  & $\mathbf{48.84_{\pm 2.90}}$
  & $\mathbf{89.72_{\pm 0.81}}$
  & $\mathbf{70.83_{\pm 3.35}}$
  & $\underline{95.56_{\pm 0.68}}$ &{\quad}
  & $\mathbf{58.71_{\pm 0.99}}$
  & $19.07_{\pm 0.70}$ \\
\midrule\hline


\multicolumn{12}{l}{\textbf{Last Acc.}}\\

\quad Primal\mbox{-}RAIL
  & \textcolor{gray}{$35.78_{\pm 1.89}$}
  & \textcolor{gray}{$59.68_{\pm 1.08}$}
  & \textcolor{gray}{$63.80_{\pm 2.79}$}
  & \textcolor{gray}{$75.47_{\pm 3.44}$}
  & \textcolor{gray}{$35.47_{\pm 9.42}$}
  & \textcolor{gray}{$56.32_{\pm 2.15}$}
  & \textcolor{gray}{$52.95_{\pm 3.64}$}
  & \textcolor{gray}{$87.71_{\pm 3.18}$} &{\quad}
  & \textcolor{gray}{$58.40_{\pm 0.86}$}
  & \textcolor{gray}{$18.06_{\pm 2.21}$}
\\\hdashline

\quad FeCAM
  & $6.32_{\pm 5.01}$ & $26.79_{\pm 10.63}$ & $19.28_{\pm 2.80}$ & $35.37_{\pm 10.16}$
  & $15.40_{\pm 13.96}$ & $12.99_{\pm 17.42}$ & $18.38_{\pm 25.53}$ & $8.13_{\pm 4.92}$ &{\quad}
  & $17.83_{\pm 5.73}$ & $11.79_{\pm 2.92}$ \\

\quad RanPAC
  & $36.98_{\pm 4.30}$
  & $51.41_{\pm 3.23}$
  & $62.19_{\pm 5.63}$
  & $74.71_{\pm 7.99}$
  & $44.69_{\pm 6.36}$
  & $75.78_{\pm 13.80}$
  & $63.11_{\pm 8.14}$
  & $87.30_{\pm 7.11}$ &{\quad}
  & $62.02_{\pm 5.68}$
  & $17.23_{\pm 2.72}$ \\

\quad KLDA
  & $\underline{43.33_{\pm 0.95}}$
  & $\mathbf{56.34_{\pm 0.91}}$
  & $\mathbf{66.43_{\pm 3.75}}$
  & $\mathbf{84.92_{\pm 1.94}}$
  & $\mathbf{47.66_{\pm 5.18}}$
  & $\underline{88.54_{\pm 2.08}}$
  & $\underline{69.52_{\pm 1.44}}$
  & $\mathbf{96.11_{\pm 0.65}}$ &{\quad}
  & $\mathbf{69.11_{\pm 0.68}}$
  & $19.51_{\pm 0.69}$ \\

\quad \textsc{HyCal}(Ours)
  & $\mathbf{44.79_{\pm 1.73}}$
  & $\underline{54.00_{\pm 1.91}}$
  & $\underline{65.91_{\pm 1.99}}$
  & $\underline{82.27_{\pm 2.04}}$
  & $\mathbf{47.66_{\pm 2.26}}$  
  & $\mathbf{90.09_{\pm 0.88}}$
  & $\mathbf{70.97_{\pm 3.49}}$
  & $\underline{95.56_{\pm 0.68}}$ &{\quad}
  & $\underline{68.91_{\pm 0.78}}$
  & $19.33_{\pm 0.61}$ \\

\bottomrule

\end{tabular}
}
\end{table*}

\section{Limitations}
\textsc{HyCal} is designed for settings in which frozen pretrained representations are already sufficiently informative and where prototype-level calibration is preferable to backbone adaptation. Accordingly, its effectiveness depends on the quality of the underlying representation space. When newly arriving tasks come from domains that lie far outside the support of the pretrained model, or when image and text embeddings are strongly misaligned, prototype calibration alone may be insufficient to recover fully discriminative structure.

In addition, although the cosine similarity and regularized Mahalanobis term improve stability in few-shot regimes, covariance estimation can still become noisy when only extremely small numbers of samples are available or when class distributions are highly multimodal. Finally, \textsc{HyCal} deliberately avoids parameter updates in order to preserve efficiency and retention, which is advantageous in data-scarce continual settings, but may limit the attainable gains in scenarios where abundant in-domain data and larger adaptation budgets make trainable methods viable.


\end{document}